%% file: main.tex
\documentclass[11pt]{article}

\usepackage[final]{acl}
\usepackage{times}
\usepackage{latexsym}
\usepackage[T1]{fontenc}
\usepackage[utf8]{inputenc}
\usepackage{microtype}
\usepackage{inconsolata}
\usepackage{booktabs}
\usepackage{amsmath,amssymb,amsthm,amsfonts}
\usepackage{mathtools}
\usepackage{verbatim}
\usepackage{graphicx}
\usepackage{multirow}
\usepackage{enumitem}
\usepackage{xspace}
\usepackage{hyperref}
\usepackage{cleveref}
\usepackage{xcolor}
\newtheorem{proposition}{Proposition}
\usepackage{overpic}
\usepackage{fontawesome5}

\newcommand{\cC}{\mathcal C}
\newcommand{\cE}{\mathcal E}
\newcommand{\cF}{\mathcal F}
\newcommand{\cG}{\mathcal G}
\newcommand{\cT}{\mathcal T}
\newcommand{\cV}{\mathcal V}

\newcommand{\bH}{H}

\newcommand{\bh}{h}
\newcommand{\bu}{u}

\newcommand{\bbN}{\mathbb{N}}

\newcommand{\stitle}[1]{\vspace{0.33ex}\noindent{\textbf{#1}}}

\newcommand{\method}{\textbf{GRAB}}

\title{Latent Bridges for Multi-Table Question Answering}

\author{
\textbf{Simone Varriale\textsuperscript{1}},
\textbf{Tamara Cucumides\textsuperscript{2}},
\textbf{Floris Geerts\textsuperscript{2}}, 
\textbf{Paolo Papotti\textsuperscript{1}} \\
\textsuperscript{1}EURECOM, Biot, France \\
\textsuperscript{2}University of Antwerp, Antwerp, Belgium \\
\texttt{\{simone.varriale, paolo.papotti\}@eurecom.fr} \\
\texttt{\{tamara.cucumidesfaundez, floris.geerts\}@uantwerp.be}
}

\begin{document}
\maketitle

\begin{abstract}
We introduce \textbf{GRAB}, a con\-struc\-tor–en\-co\-der–bridge pipeline for table question answering. Our method lifts relational data into a heterogeneous graph, encodes it via message passing, and transfers the  signals to an LLM through a small set of query-conditioned latent tokens. The LLM receives these tokens alongside the serialized table, and the two play different roles: the graph shows where the relevant evidence is, while the text keeps the exact cell values that the graph does not encode. The graph supplements the text rather than replacing it, and our ablations confirm that both are required.
Crucially, the LLM remains strictly frozen to preserve its general reasoning capabilities; we train only the lightweight graph encoder and latent bridge (91M parameters), allowing the entire pipeline to be trained efficiently. Our pipeline significantly improves performance on relational Question Answering, with the largest gains in demanding multi-table settings, offering an efficient, principled way to connect relational deep learning with LLMs.
{\small \faGithub~\href{https://github.com/simonevarriale/GRAB}{Link to Code}}
\end{abstract}

\section{Introduction}

Table question answering (TQA) asks language models to answer natural language (NL) questions grounded in structured data.
While Text-to-SQL is popular for querying  databases, TQA is essential because SQL struggles with messy, unnormalized data, implicit relationships, or hybrid contexts where tables are mixed with free text \cite{badaro23}. 
Most LLM-based approaches treat tables as text: they serialize rows and columns into a 1D sequence and rely on the model to recover the underlying structure. This strategy is convenient, but it mismatches the semantics of tables, where row--column organization, permutation invariance, hierarchical headers, and cross-cell dependencies are central to meaning. 
This loss of structure is a key reason why LLMs remain brittle on table reasoning \cite{tamo}.

A natural alternative is to treat tables as a separate modality and interface them with LLMs through learned representations rather than raw serialization alone. In this work, we encode tables with a dedicated neural network and inject the resulting features into an LLM as latent tokens. 
Our starting point is that many of the difficulties of TQA are inherently relational: relevant evidence is distributed across rows, columns, and value groups, and in the multi-table setting the model must reason across linked tables through join dependencies.

We therefore propose \textbf{GRAB} (\textbf{G}raph-\textbf{R}elational \textbf{A}ttention \textbf{B}ridge), a GNN-based table encoder for LLM-based \emph{multi-table} question answering. As shown in Figure \ref{fig:architecture}, our encoder lifts relational data into a graph, uses message passing to capture structural dependencies, and compresses the result into a small set of latent tokens consumed by the LLM. The representation is conditioned on the NL question, 
allowing the encoder to produce question-relevant structure rather than a single static summary of the input table. 

Adapting LLMs to tabular data typically requires computationally expensive full fine-tuning or parameter-heavy adapters, e.g., LoRA \cite{HuSWALWWC22}. While LoRA freezes the pretrained backbone and learns task-specific low-rank updates, it alters the model's internal representations. This can cause the adapted inference behavior to over-specialize to the fine-tuning domain, potentially reducing out-of-distribution cross-task performance and causing catastrophic forgetting \cite{huang2024lorahub}. In contrast, our approach keeps the LLM strictly frozen and isolates the task-specific learning entirely within our 91M parameter external module.
This lightweight design allows the entire pipeline to be trained efficiently on a single GPU, democratizing multi-table reasoning.

\begin{figure*}[t] 
    \centering



\includegraphics[width=0.95\textwidth]{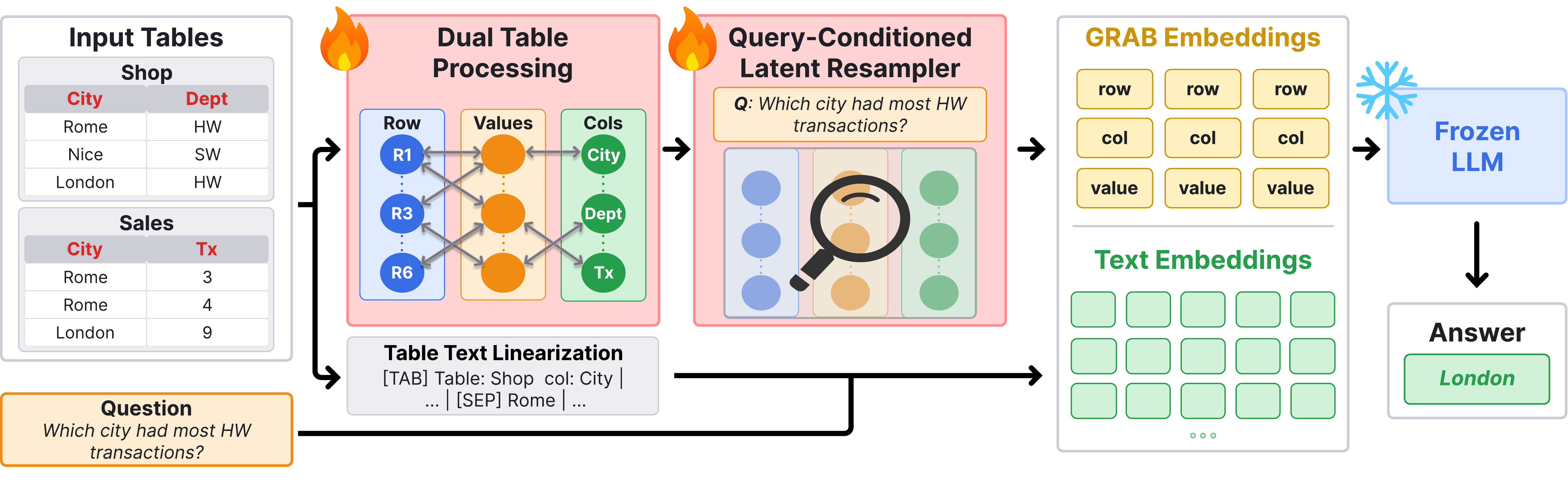}
\caption{\textbf{Architecture Overview.} Tables are processed via two parallel streams: text serialization and a tripartite graph that explicitly captures multi-table joins. A \textit{Query-Conditioned Latent Resampler} uses the natural language question to actively filter the GNN-encoded graph into dynamic soft tokens. These structural tokens guide a frozen LLM to generate the final answer. Only the lightweight graph and resampler modules are updated during training.}
    \label{fig:architecture}
\end{figure*}

Our central claim is that graph-conditioned latent interfaces provide a practical middle ground between two extremes: pure text serialization, which underuses relational structure, and symbolic pipelines, which often sacrifice the flexibility of LLMs on underspecified or compositional questions. By combining a relational graph encoder with a lightweight latent bridge, we preserve structural bias while keeping the downstream model fully compatible with autoregressive LLM reasoning. Empirically, we show that this design is especially effective on multi-table and structurally demanding questions. Conceptually, our results confirm that, for TQA, tables should not be treated merely as text, but as a structured modality that requires its own encoder and interface to the LLM.

In summary, our main contributions are threefold. First, we introduce \textbf{GRAB}'s architecture (Section \ref{sec:method}) and formalize it (Section \ref{sec:theory}). Second, we design a stress-test taxonomy  that isolates structural evidence localization from exact arithmetic, providing a novel diagnostic tool for LLMs 
(Section \ref{sec:expSetup}).
Finally, we show that \textbf{GRAB} consistently outperforms serialization-only 
across 13 single- and multi-table QA benchmarks (Section \ref{sec:expResult}). 

\section{Problem Setting}
\label{sec:problem_setting}
We formulate multi-table question answering (TQA) as a conditional generation task over a relational database\footnote{The proposed solution works also for other tasks, such as tabular fact checking. We report these results in Appendix \ref{app:factcheck}.}. Let $\mathcal{T} = \{T_1, T_2, \dots, T_n\}$ be a set of tables, where each table $T_i$ consists of a set of columns $C_i$, rows $R_i$, and cell values $V_i$. The database is accompanied by metadata $\mathcal{M}$, which defines the schema, including foreign-key (FK) relationships ($\mathcal F$) linking tables. Columns have a type $\tau\in\{\mathrm{cat(egorical)},\mathrm{num(erical)},\mathrm{text}\}$.

Given a natural language question $Q$ and the relational context $(\mathcal{T}, \mathcal{M})$, the goal is to generate a target answer $A$, which may be a free-form generative text, an extractive span, or a numeric aggregate. 

In the standard setting, an LLM is either used in a zero/few-shot setting (i.e., \textit{frozen}) or trained to maximize the probability of the correct answer $P(A \mid Q, \mathcal{T}, \mathcal{M})$. Typically, the input tables are flattened into a text sequence $S_{\text{\sl table}}$, and the LLM implicitly reconstructs row-column alignments and cross-table linkages \cite{badaro23}. 
However, tabular data exhibits unique semantic properties, such as permutation invariance of rows and strict hierarchical header structures, that are distorted or even lost by serialization.

\section{Related Work}
\textbf{LLMs for Table Question Answering.}
Early approaches to TQA rely on encoder-only models pretrained on flattened tabular data \cite{tapas,tapex}. 
Other models mitigate the loss of structure caused by serialization with specialized attention biases to capture row-column alignments and preserve permutation invariance \cite{tableformer}, but injecting these structural biases requires retraining of the LLM. Indeed, with LLMs, the paradigm shifted to text serialization \cite{unifiedskg}, where tables are converted into (Markdown or HTML) sequences and processed via standard autoregressive generation \cite{tablellama}. However, current pure-LLM TQA strategies suffer context-window fragmentation when serializing multiple tables, losing structural coherence \cite{griqa,chen-etal-2024}. 

Recent models such as TAMO \cite{tamo} inject hypergraph-encoded tables as soft tokens. TAMO constructs a hypergraph over cell occurrences: each cell is a primitive node, while rows, columns, and the whole table act as hyperedges. 
In contrast, \method{} canonicalizes repeated values within column classes and merges foreign-key-linked column occurrences into shared classes. Equality patterns and join keys therefore become explicit graph connectivity rather than implicit textual or embedding-level coincidences. Moreover, TAMO's encoding is query-agnostic, whereas our latent bridge is conditioned on the NL question. 

\stitle{Graph Learning for Relational Data.}
Tabular ML is dominated by tree-based models \cite{xgboost} and multilayer perceptrons, which treat rows as independent, identically distributed samples. Tabular foundation models \cite{tabpfn,trl25} introduce cross-row attention but are designed for row-level predictions. For TQA, we argue that relational databases are too dense to be losslessly compressed as static tokens. 

Conversely, relational deep learning \cite{relbench} explicitly models multi-table databases as heterogeneous graphs, using message passing to capture foreign key links. While these models excel at node classification over databases, they are rarely integrated into LLM pipelines for TQA. Our work bridges this gap by using graph constructors to encode tabular data before reasoning.

\stitle{
Soft Tokens \&
Multimodality.}
Parameter-efficient fine-tuning methods \cite{prefix,liu22} introduced ``soft tokens'': continuous, learnable prompt vectors that steer frozen LLMs without updating their weights. This paradigm has been adapted for multimodal bridging, where models use latent resamplers to compress continuous signals from vision/audio encoders into few soft prefix tokens \cite{flamingo, blip2}. 
We are the first to use query-conditioned latent resamplers to compress \textit{relational structures}. Rather than acting as a generic summary, our latents act as a learned structural retrieval bridge (see Appendix~\ref{app:method-comparison}, Table \ref{tab:method-comparison} for a feature comparison).

\section{Method}
\label{sec:method}

\input{sec4-short}

\stitle{Graph Encoder.}
Given the initialized node representations
$H_R^{(0)}$, $H_C^{(0)}$, and $H_V^{(0)}$, the graph encoder applies message passing over the tripartite incidence structure of $\mathcal{G}$. The encoder maintains separate hidden states for row, column, and value nodes, and updates them through the row-value and column-value adjacency matrices. 
Let $A_{RV} \in \{0,1\}^{|\cV_R| \times |\cV_V|}$ denote the row-value incidence matrix, where $(A_{RV})_{ig}=1$ if row node $r_i$ is
connected to value node $v_g$. Similarly, let
$A_{CV} \in \{0,1\}^{|\cV_C| \times |\cV_V|}$ denote the column-value incidence matrix, where $(A_{CV})_{jg}=1$ if column node $c_j$ is
connected to value node $v_g$.

At layer $\ell$, value nodes aggregate messages from their
incident row and column nodes, while row and column nodes aggregate messages
from their incident value nodes:
\[
M_V^{(\ell)} =
A_{RV}^{\top} H_R^{(\ell)}
+
A_{CV}^{\top} H_C^{(\ell)},
\]
\[
M_R^{(\ell)} = A_{RV} H_V^{(\ell)}, \qquad
M_C^{(\ell)} = A_{CV} H_V^{(\ell)}.
\]

The row, column, and value states are then updated in parallel using type-specific residual blocks. For each node type $X$, where $X$ may denote
rows $R$, columns $C$, or values $V$, the update is:
\[
\begin{aligned}
\tilde H_X^\ell
  &= H_X^\ell + \mathrm{Drop}\!\left(N_X^{\mathrm{msg}}(M_X^\ell)\right),\\
H_X^{\ell+1}
  &= \tilde H_X^\ell
   + \mathrm{Drop}\!\left(F_X\!\left(N_X^{\mathrm{ffn}}(\tilde H_X^\ell)\right)\right),
\end{aligned}
\]
where $\mathrm{Drop}$, $N_X^{\mathrm{msg}}$, $N_X^{\mathrm{ffn}}$, and $F_X$ denote dropout, message normalization, FFN normalization, and the feed-forward network, respectively.
    The update follows a Transformer-style residual design: the first residual branch injects aggregated graph messages, while the second residual branch applies a type-specific feed-forward refinement to the message-updated node state. Each node type has separate normalization and feed-forward parameters. This message-passing scheme lets row and column representations exchange information through value nodes. As a result, rows that contain the same canonical value group can influence one another, columns receive signals from the values they generate, and in the multi-table case, shared value or foreign-key-linked structures allow information to propagate across tables.
After $L$ graph layers, the encoder outputs contextualized node representations $H_R^{(L)}$, $H_C^{(L)}$, and $H_V^{(L)}$. These representations are not pooled
into a single graph vector, instead, they are passed to the latent bridge, which compresses the variable-size graph into a fixed number of soft tokens for the
LLM.

\stitle{Latent Bridge to the LLM.}
In standard prompt- and prefix-tuning \cite{prefix}, the LLM is steered by a set of task-specific soft tokens that remain static at inference time: the same prompt vectors are prepended regardless of the specific input instance. In our TQA setting, however, the structural context changes dynamically based on both the input graph $\mathcal{G}$ and the user question $Q$. Therefore, we design a \textit{query-conditioned latent resampler} based on the style of a Perceiver Resampler to generate a sequence of soft tokens for every forward pass \cite{flamingo}.

Let $H_{\text{\sl graph}} \in \mathbb{R}^{N \times d}$ denote all final node embeddings output by the GNN, and $H_Q \in \mathbb{R}^{M \times d}$ denote the contextualized embeddings of the NL question (e.g., via a lightweight RoBERTa model). To bridge the modality gap without exhausting the LLM's context window, we initialize a fixed number of $K$ learnable query vectors $Z^{(0)} \in \mathbb{R}^{K \times d}$. 

To make the soft tokens question-conditioned, the resampler uses the question
embeddings $H_Q$ as additional context during latent extraction. The initial
learnable latent queries are partitioned by node type as
$Z^{(0)} = [Z_R^{(0)}; Z_C^{(0)}; Z_V^{(0)}]$, where
$Z_R^{(0)} \in \mathbb{R}^{K_R \times d}$,
$Z_C^{(0)} \in \mathbb{R}^{K_C \times d}$, and
$Z_V^{(0)} \in \mathbb{R}^{K_V \times d}$ correspond to row, column, and value latents, respectively. The question is encoded once to obtain $H_Q$ and is then reused by the latent groups. Each group performs cross-attention over its corresponding graph node representations concatenated with the same question embeddings along the sequence dimension:
$Z_R = \mathrm{Attn}(Z_R^{(0)}, [H_R; H_Q])$,
$Z_C = \mathrm{Attn}(Z_C^{(0)}, [H_C; H_Q])$, and finally
$Z_V = \mathrm{Attn}(Z_V^{(0)}, [H_V; H_Q])$.

The output is a sequence of $K$ soft tokens $Z \in \mathbb{R}^{K \times d}$, linearly projected to match the hidden dimension $d_L$ of the frozen LLM. Unlike static prefix tokens, these $K$ latents 
compress the multi-table graph into a summary conditioned on the question.

\stitle{LLM Interface and Answer Generation.}
The latent bridge outputs a fixed-length sequence of graph-derived soft tokens
$Z = \{z_1,\ldots,z_K\}$ in the graph encoder hidden space. Before being passed to the LLM, these vectors are mapped to the LLM embedding dimension through a
learned projector: $\hat{Z} = \mathrm{Proj}(Z)$.
The projected latents $\hat{Z}$ are then injected as soft prefix embeddings before the prompt.

The final LLM input consists of the graph latents, the task description, the
serialized textual table context, and the NL question: $[\hat{Z}; D; S_{\mathrm{table}}; Q]$, where $D$ denotes the instruction or dataset-specific description, $S_{\mathrm{table}}$ denotes the retained textual serialization of the table or table segments, and $Q$ is the question. Thus, the LLM receives both explicit
textual context and a compact 
structural prefix. 

\stitle{Training Objective.} Our training paradigm is designed to be highly parameter-efficient. Let $\Theta_{\text{\sl LLM}}$ denote the parameters of the large language model, and $\Phi_{\text{\sl GNN+Bridge}}$ denote the parameters of our graph encoder and query-conditioned resampler ($\sim$91M parameters). During training, we freeze $\Theta_{\text{\sl LLM}}$ and optimize only $\Phi_{\text{\sl GNN+Bridge}}$ to minimize the standard auto-regressive negative log-likelihood of the target answer $A$:
$
    \mathcal{L} = - \sum_{t=1}^{|A|} \log P_{\Theta_{\text{\sl LLM}}}\big(a_t \mid a_{<t}, Q, \hat{Z}, S_{\text{\sl table}}\big)
$ 
By keeping the gradients entirely within the lightweight $\Phi_{\text{\sl GNN+Bridge}}$ modules, the computational memory footprint is drastically reduced, allowing the framework to be trained end-to-end on a single GPU while fully preserving the LLM's pre-trained general knowledge.

\section{Structural Analysis}
\label{sec:theory}

\stitle{Why Explicit Graph Construction Helps?}
A serialized table represents equality, co-occurrence, and joins only indirectly, as repeated strings scattered across a sequence. The  constructor turns these relations into graph structure.
Repeated categorical or textual values become shared value nodes; foreign-key-linked columns become shared column classes. As a result, duplicate elimination becomes a value-node degree property, and joins become  bounded-hop paths through shared value groups. The
constructor therefore makes common relational operations directly available to
the encoder.





\stitle{What Message Passing Can Extract.}
The encoder is a typed message-passing network over row, column-class, and value-group nodes.
Each layer expands the accessible neighborhood by one graph
hop, so an $L$-layer encoder extracts bounded-depth structural features of
$\cG$. In \method{}, these features include row membership, column membership,
repeated-value support, row-level co-occurrence, and local foreign-key
connectivity. The encoder is most useful when the bottleneck is structural
access, e.g., locating rows, forming groups, detecting repeated values, or following
joins. It is less useful when the relevant evidence has already been located
and the remaining difficulty is exact symbolic or numerical computation.


\stitle{Limits of the Latent Bridge.}
The latent bridge is a readout and compression mechanism, not an additional source of graph expressivity. It receives the node representations produced by
message passing and turns them into a fixed number of soft tokens.
It can select, weight, and summarize information exposed by the graph encoder, but it cannot recover distinctions that  were not represented in the constructed graph or were erased during message passing.

This fixed-capacity compression creates a bottleneck. A question-agnostic
bridge must compress the whole graph into the same sketch regardless of what is
being asked, so it may waste capacity on irrelevant details or discard facts
needed for a particular query. A question-conditioned bridge mitigates this by
using the question to decide which rows, columns, values, and joins should be
emphasized. Thus, it does not make the encoder more expressive, but it makes the
limited latent capacity more useful for the current question. Appendix~\ref{app:logical-bridge}
formalizes this intuition as a sketch-complexity separation.

\section{Experimental Setup}
\label{sec:expSetup}
\stitle{Datasets.}
We evaluate our method on a suite of benchmarks covering both single- and multi-table QA, as listed in Table \ref{tab:datasets}.

\begin{table}[t]
\centering
\small
\setlength{\tabcolsep}{3.5pt}
\begin{tabular}{c l c c l}
\toprule
 & \textbf{Dataset} & \textbf{Extra} & \textbf{Train} & \textbf{Main Challenge} \\
\midrule
\multirow{6}{*}{\rotatebox[origin=c]{90}{\textbf{Single-table}}}
& StructQA    & No     & 4.5k  & Structure sensitivity \\
& HiTab       & No     & 7.4k  & Hierarchical tables \\
& WTQ         & No     & 11.3k & Compositional QA \\
& WikiSQL     & No     & 56.4k & Filtering \& aggregation \\
& HCTQA       & Layout & 62.1k & Complex layouts \\
& TabMWP      & Text   & 23.1k & Math reasoning \\
\midrule
\multirow{7}{*}{\rotatebox[origin=c]{90}{\textbf{Multi-table}}}
& MultiHierTT & Text   & 7.0k  & Multi-hop reasoning \\
& SciTaT      & Text   & 11.6k & Scientific evidence \\
& MMQA        & Text   & 2.3k  & Cross-modal reasoning \\
& TQA-Bench   & No     & 9.8k  & Long-context joins \\
& Atis        & No     & 0.4k  & Flight-query reasoning \\
& GeoQuery    & No     & 0.5k  & Geographic querying \\
& Spider      & No     & 6.0k  & SQL result generation \\
\bottomrule
\end{tabular}
\caption{Datasets used in our experiments. The suite spans standard single-table QA, structure-heavy and layout-heavy settings, and multi-table or hybrid benchmarks requiring reasoning across tables and text.}
\label{tab:datasets}
\end{table}

For single-table experiments, we use six datasets. StructQA \cite{tamo} focuses on
table structure understanding and robustness to structural variation.
HiTab \cite{cheng-etal-2022-hitab} emphasizes hierarchical headers and aggregation-heavy reasoning.
WTQ \cite{wtq} and WikiSQL \cite{wikisql} provide standard flat-table benchmarks with broad use in
the literature, while HCTQA \cite{hctqa} focuses on human-centric tables with complex
layouts. TabMWP \cite{tabmwp} complements these datasets with table-grounded math word
problems that require numerical and multi-step reasoning.

For multi-table experiments, we use seven datasets from reasoning with relations to hybrid
QA over tables and text.
MultiHiertt \cite{zhao-etal-2022-multihiertt}, MMQA \cite{mmqa} and SCITAT \cite{zhang-etal-2025-scitat} combine tables with textual evidence, TQA-Bench \cite{qiu2024tqabenchevaluatingllmsmultitable} focuses on scalable multi-table relational reasoning under long contexts (8K), and Atis, GeoQuery and Spider \cite{pal-etal-2023-multitabqa} target QA where the output itself may be tabular.

\stitle{Baselines}
We compare our method, \textbf{GRAB},  
against representative baselines. We include an \textbf{Inference-only} baseline, where the model receives only the serialized table(s) and question in zero-shot form, without any trainable table encoder or prompt parameters. 
For the \textit{Frozen LLM} setting, we compare against \textbf{Prompt Tuning}, where the base LLM remains fixed and only a small set of learned soft prompt vectors is optimized, and \textbf{TAMO}\footnote{TAMO is excluded from our multi-table experiments as the released implementation supports only single-table inputs.}, which encodes each table with a hypergraph neural network and injects the resulting latent table features into the frozen LLM as soft tokens. 
Finally, we consider a \textit{Tuned LLM} setting, where the LLM is adapted with LoRA with and without the GRAB encoder. As an additional reference, we report results for \textbf{GPT-5.4-mini} \cite{openai2026gpt54mini} under an inference-only setup.
We also include three task-specialized baselines: \textbf{TableLlama} \cite{tablellama} and \textbf{TaBERT} \cite{yin-etal-2020-tabert} for the single-table setting and \textbf{MultiTabQA} \cite{pal-etal-2023-multitabqa} for the multi-table setting. These baselines are fine-tuned separately on each dataset, providing supervised references for the corresponding evaluation regimes.




\begin{table*}[t]
\centering
\small
\setlength{\tabcolsep}{6pt}
\begin{tabular}{llccccccc}
\toprule
\textbf{Setting} & \textbf{Method}
& \multicolumn{1}{c}{\textbf{StructQA}}
& \multicolumn{1}{c}{\textbf{HiTab}}
& \multicolumn{1}{c}{\textbf{WTQ}}
& \multicolumn{1}{c}{\textbf{WikiSQL}}
& \multicolumn{2}{c}{\textbf{HCTQA}}
& \multicolumn{1}{c}{\textbf{TabMWP}} \\
\cmidrule(lr){3-3}\cmidrule(lr){4-4}\cmidrule(lr){5-5}\cmidrule(lr){6-6}\cmidrule(lr){7-8}\cmidrule(lr){9-9}
 & 
 & \textbf{Acc.}
 & \textbf{Acc.}
 & \textbf{Acc.}
 & \textbf{Acc.}
 & \textbf{F1}
 & \textbf{CC}
 & \textbf{Acc.} \\
\toprule
\multirow{1}{*}{Inference-only}
& zero-shot & 24.26 & 53.72 & 32.29 & 53.86 & 30.16 & 14.06 & 34.75 \\
\midrule
\multirow{3}{*}{Frozen LLM}
& Prompt tuning & 71.33 & 69.63 & 56.02 & 85.71 & 68.88 & 43.09 & 84.75 \\
& TAMO & 73.40 & 65.27 & 56.95 & 86.29 & 70.89 & 47.70 & 84.49 \\
& \textbf{GRAB} & \textbf{84.80} & \textbf{74.49} & \textbf{58.33} & \textbf{88.30} & \textbf{86.38} & \textbf{72.81} & \textbf{87.01} \\
\midrule
Tuned LLM & zero-shot & 77.93 & 76.01 & 57.51 & 90.17 & 86.94 & 72.53 & \textbf{88.04} \\
(LoRA) & \textbf{GRAB} & \textbf{86.00} & \textbf{76.83} & \textbf{58.75} & \textbf{90.23} & \textbf{87.93} & \textbf{75.66} & 87.78 \\
\midrule
\multirow{3}{*}{Others}
& TableLLama & 72.00 & 63.82 & 48.59 & 87.37 & 68.34 & 43.49 & 80.05 \\
& TaBERT  & 81.40 & 71.97 & 55.79 & 87.56 & 83.44 & 66.86 & 87.04 \\
& GPT-5.4-mini & 48.46 & 72.28 & 65.88 & 72.99 & 57.13 & 36.60 & 72.06 \\
\bottomrule
\end{tabular}
\caption{Single-table question answering results with Qwen3-4B-Base as the base model. We report denotation accuracy on all datasets, except F1 and Complete Containment (CC) on HCTQA.}
\label{tab:single_table_main}
\end{table*}

\begin{table*}[t]
\centering
\small
\setlength{\tabcolsep}{3pt}
\resizebox{\textwidth}{!}{%
\begin{tabular}{llccccccccccc}
\toprule
\textbf{Setting} & \textbf{Method}
& \multicolumn{1}{c}{\textbf{MultiHiertt}}
& \multicolumn{2}{c}{\textbf{SCITAT}}
& \multicolumn{1}{c}{\textbf{MMQA}}
& \multicolumn{1}{c}{\textbf{TQA}}
& \multicolumn{2}{c}{\textbf{Atis}}
& \multicolumn{2}{c}{\textbf{GeoQuery}}
& \multicolumn{2}{c}{\textbf{Spider}} \\
\cmidrule(lr){3-3}
\cmidrule(lr){4-5}
\cmidrule(lr){6-6}
\cmidrule(lr){7-7}
\cmidrule(lr){8-9}
\cmidrule(lr){10-11}
\cmidrule(lr){12-13}
&
& \textbf{Acc.}
& \textbf{EM} & \textbf{F1}
& \textbf{Acc.}
& \textbf{Acc.}
& \textbf{T-EM} & \textbf{Cell F1}
& \textbf{T-EM} & \textbf{Cell F1}
& \textbf{T-EM} & \textbf{Cell F1} \\
\toprule
\multirow{1}{*}{Inference-only}
& zero-shot & 7.18 & 6.18 & 18.67 & 3.34 & 38.19 & 1.16 & 16.89 & 4.74 & 12.72 & 2.92 & 39.27 \\
\midrule
\multirow{2}{*}{Frozen LLM}
& Prompt tuning & 16.38 & 23.24 & \textbf{41.49} & 25.11 & 84.14 & 70.93 & 64.09 & 43.87 & 36.28 & 34.35 & 71.58 \\
& \textbf{GRAB} & \textbf{21.55} & \textbf{26.18} & 41.38 & \textbf{31.10} & \textbf{91.29} & \textbf{79.07} & \textbf{64.78} & \textbf{51.78} & \textbf{43.12} & \textbf{36.86} & \textbf{75.11} \\
\midrule
Tuned LLM & zero-shot & \textbf{22.99} & 19.12 & 36.17 & 28.70 & \textbf{91.67} & 86.05 & 78.35 & \textbf{64.82} & \textbf{46.82} & 36.44 & 67.91 \\
(LoRA)& \textbf{GRAB} & 22.51 & \textbf{26.47} & \textbf{37.42} & \textbf{29.42} & 90.05 & \textbf{86.05} & \textbf{79.49} & 63.24 & 45.84 & \textbf{39.50} & \textbf{71.97} \\
\midrule
\multirow{2}{*}{Others}
& MultiTabQA & 6.51 & 5.00 & 34.81 & 27.27 & 81.19 & 61.63 & 34.00 & 54.15 & 41.19 & 27.68 & 62.35\\
& GPT-5.4-mini & 21.84 & 32.65 & 30.16 & 26.66 & 54.24 & 2.33 & 36.30 & 36.36 & 38.24 & 24.62 & 57.52 \\
\bottomrule
\end{tabular}%
}
\caption{Multi-table question answering results with Qwen3-4B-Base. We report denotation accuracy on MultiHiertt and MMQA, EM/F1 on SCITAT, accuracy on TQA-Bench, and Table Exact Match (T-EM) and Cell F1 on Atis, GeoQuery and Spider.}
\label{tab:multi_table_main}
\end{table*}

\stitle{Stress-test Taxonomy.}
Existing benchmarks conflate three orthogonal sources of difficulty: the form of the expected answer, the structural depth of table access required to locate evidence, and the computational path needed to derive the final value. Appendix~\ref{app:structural-hierarchy} formalizes this as a three-axis taxonomy (answer type, structural depth, and computational path), which lets us attribute failures to a specific axis rather than to aggregate ``hardness.'' To isolate the contributions of \method{}, we design a targeted stress-test over 15 relational tables that independently varies the \textbf{structural axis} (0--2 categorical filters, plus optional \textsc{group-by} keys) and the \textbf{computational axis} (\textsc{lookup}, \textsc{count}, \textsc{max}, \textsc{avg}). 

\stitle{Implementation Details.}
Our reference backbone is {Qwen3-4B-Base} \cite{qwen3}. For the graph encoder, we initialize textual row and column representations with embeddings from {Qwen3-Embedding-0.6B}. All experiments, including both inference and fine-tuning, are run on {NVIDIA A100} GPUs. The encoder--bridge component has 91M parameters in our chosen configuration. 
Appendix \ref{app:implementation-details} reports the full training setup for GRAB and the baselines (optimizer, schedule, batch size, LoRA configuration, GNN hyperparameters) as well as dataset preprocessing details.

\stitle{Evaluation Metrics.}
We report the original metric from each benchmark. For most datasets it is accuracy: a prediction is correct only if, after normalization, its multiset of answer values exactly matches the gold answers. We parse both the prediction and the reference into a multiset of values, apply case-folding and numeric canonicalization (comma stripping and integer/float unification), and require multiset equality. The comparison is order-invariant but count-sensitive, partial overlap earns no credit. 
For HCTQA we report F1 together with Complete Containment (CC), a binary score that is 1 only when the prediction fully covers the gold answer set (recall = 1). As TQA-Bench answers are multiple-choice, we report accuracy by extracting the choice from the generated text before comparing it to the gold option. For SciTaT, we follow its protocol and split examples by answer length: short-form answers are scored by Exact Match, while free-form answers are scored by token-level F1. For the SQL-style multi-table benchmarks (ATIS, GeoQuery, Spider-SQL), the model generates the answer as a linearized result table, which we evaluate by table-level exact match (T-EM), 
i.e., full ordered table must match, and by cell-level F1, which scores cells as unordered multisets.  


\section{Main Results}
\label{sec:expResult}

Tables~\ref{tab:single_table_main} and~\ref{tab:multi_table_main} report results across single- and multi-table benchmarks. By comparing \method{} against serialization, graph, and scale baselines, four patterns stand out:
\textbf{(i) Massive gains on demanding structures.} \method{} consistently outperforms serialized baselines, with the largest jumps exactly where 1D text struggles most: complex layouts (HCTQA: $+17.5$ F1) and multi-table joins (Spider: $+13.3$ Cell F1; TQA: $+7.1$ Acc). This confirms that modeling tables relieves the LLM from implicitly reconstructing relational structure.
\textbf{(ii) Frozen \method{} rivals LoRA fine-tuning.} Despite training only ${\sim}91$M parameters while keeping the LLM frozen, \method{} approaches or beats a LoRA-adapted LLM reading serialized text (e.g., $84.80$ vs $77.93$ on StructQA). Injecting structural bias proves more efficient than updating LLM weights to brute-force structure from text. Combining both yields the strongest results overall.
\textbf{(iii) Query-conditioned graphs beat static hypergraphs.} \method{} outperforms TAMO on every single-table benchmark: dynamically allocating latent capacity based on the question is superior to query-agnostic compression.
\textbf{(iv) Punching above its weight class.} \method{} consistently outperforms fine-tuned, TQA-specialized models (TableLlama and MultiTabQA). Our 4B-parameter pipeline rivals or exceeds GPT-5.4-mini (e.g., $+29.2$ F1 on HCTQA, $+37.0$ Acc on TQA-Bench), a closed-source model roughly 100$\times$ larger, demonstrating that explicit relational encoding can bridge massive gaps in raw parametric scale.

\input{tables/tab_stress-summary}

\stitle{Performance by Query Type.} Table~\ref{tab:stress-summary} reports a condensed view of the diagnostic tests 
(full results in Appendix~\ref{app:stress-test}), revealing how \method{}'s topology mitigates structural bottlenecks while exposing inherent LLM arithmetic limits. 
On \textbf{locating evidence} (\textsc{lookup}), \method{} gains $+14$ to $+24$ F1 across all condition depths, as the graph explicitly links rows via shared value nodes to route condition-matching information. This structural advantage is most pronounced on \textbf{counting}: while a serialized baseline must reconstruct frequencies from sequence position, cardinality in our graph reduces to a local structural property (a value node's degree), yielding our largest absolute gains (up to $+46$ F1). Similarly, for \textbf{group-by} operations, \method{}'s explicit column-class nodes act as natural partition anchors, substantially recovering performance (e.g., Group-by \textsc{max}: $42.9 \to 72.6$ F1). Conversely, \textbf{exact arithmetic} (\textsc{avg}) defines the ceiling of our structural bridge. Because our graph maps continuous data to quantile buckets, it cannot execute exact math; it merely isolates the correct rows for the frozen LLM. Consequently, \textsc{avg} remains the hardest operator, with gains appearing only when filtering shrinks the operand set. These patterns are not an artifact of backbone size: rerunning the stress-test with a larger LLM (Qwen3-14B, Table~\ref{tab:stress_test_14b}) shows the \method{} gap widens rather than closes, the \textsc{avg} ceiling is preserved exactly, and the serialized baseline does not improve on Group-by despite the larger model: the bottleneck is structural, not parametric.

\stitle{Ablations.}
\textit{GNN Depth and Latent Count.} A sequential search over latent
count $K \in \{1, \dots, 256\}$ and GNN depth $L \in \{1, \dots, 16\}$ (Appendix~\ref{app:full-ablation}) shows that $K=32$ suffices and that
deeper encoders yield negligible improvement over $L=1$. 
This is consistent with the tripartite structure: row-node initialization encodes per-row cell content, value nodes act as sinks for rows sharing an attribute value, and column nodes aggregate their values, so a single message-passing step exposes co-occurrence and column membership to every value node. 


\noindent  \textit{Question Conditioning Variants.} Removing question conditioning
from the resampler costs $1$--$3$ points on three of four ablation benchmarks (Appendix~\ref{app:full-ablation}, Table \ref{tab:q_resampler_table_prompt_ablation}). The effect is modest but consistent in direction on multi-table benchmarks, supporting
the sketch-complexity argument (Appendix~\ref{app:logical-bridge}) that question conditioning helps most where the same graph must answer diverse queries. 

\noindent \textit{Serialized Table.} Removing the serialized table text from the prompt causes a collapse across datasets, e.g., 84.80 to 10.07 on StructQA (Appendix~\ref{app:full-ablation}, Table \ref{tab:q_resampler_table_prompt_ablation}), confirming that \textbf{GRAB} acts as a structural supplement rather than replacing the textual evidence needed by the LLM for exact values.

\noindent \textit{Additional Ablations.} Appendix~\ref{app:full-ablation} reports further checks on architectural choices and robustness. First, a linear projection head matches or outperforms deeper MLP bridges, indicating that the graph encoder and latent resampler already perform the relevant structural abstraction. Second, results across five random seeds show low variance, confirming that the gains are stable rather than driven by a favorable initialization.

\input{tables/table-overhead-wide}

\noindent \textit{Robustness and Generalization.}
We conduct additional experiments to assess whether \method{}'s gains depend
on the LLM backbone, context length, serialization format, or numerical
bucketing scheme. Across Gemma, Llama, and Qwen backbones ranging from 3B
to 27B parameters, \method{} consistently improves over the corresponding
frozen serialization-only baselines and generally outperforms soft-prompt
tuning. On TQA-Bench, models trained at the 8K scale retain a substantial
advantage over the frozen baseline when evaluated on longer contexts, up to
128K with Qwen3.5-4B, although performance gradually decreases as context
length increases. Changing the serialization format at evaluation time
causes dataset-dependent degradation, but \method{} remains above the
corresponding frozen baseline, while format-specific training largely
recovers performance. We further evaluate an adaptive numerical bucketing
scheme in which the number of bins scales with column cardinality; its
performance is broadly comparable to the fixed 10-bin configuration, with
small gains on some datasets and small declines on others. Full results and analysis are reported in
Appendix~\ref{app:full-ablation}.

\subsection{Computational Overhead}
\label{sec:computational_overhead}

We measure inference latency, FLOPs, GPU memory, and preprocessing cost on the full StructQA test set (1,500 queries), comparing the serialization-only frozen LLM against GRAB using the same frozen Qwen3-4B, batch size 1, and CUDA-event timing. Results are shown in Table \ref{tab:computational_overhead}.

Latency, memory, preprocessing time, and parameter counts are measured
directly. FLOPs are estimated analytically: LLM compute follows the
standard $2N_{\mathrm{params}}N_{\mathrm{tokens}}$ approximation,
while GNN and resampler costs are derived from their matrix operations
at the mean observed graph sizes. Since batch-1 latency is dominated by
autoregressive decoding, we additionally report milliseconds per
generated token. Both systems produce complete answers, with a mean generated length of 23.9 tokens for the baseline and 21.7 for GRAB. Although GRAB incurs an $8.0\%$ higher per-token decoding cost and $12.2\%$ more FLOPs due to processing 96 latent tokens, it is faster end-to-end because it generates more concise answers (median 7 vs.\ 10 tokens). We therefore interpret the end-to-end speedup as a behavioral effect of training rather than an architectural efficiency gain. The overhead introduced by the architecture is therefore the (8.0\%) increase in per-token decoding cost, which is consistent with the analytically estimated (12.2\%) increase in FLOPs.

GNN message passing requires only 0.0010 TFLOPs/query
(0.016\% of LLM prefill compute), the query-conditioned resampler
requires 0.0055 TFLOPs (0.088\%), and the projector requires
0.0005 TFLOPs. Thus, the graph-specific modules themselves account for
approximately 0.1\% of LLM prefill compute. Nearly all additional FLOPs
arise from prefilling the fixed set of 96 soft tokens. Because this cost
is constant while serialized prompts grow with table size, its relative
contribution decreases for larger tables.

Graph preprocessing, including node and edge construction, quantile
bucketing, and node-feature extraction with the embedding
model, takes a median of 53.8 ms per sample and produces a 0.10 MB
artifact. At this batch-1 rate, preprocessing all 7,500 StructQA
samples takes approximately 7 minutes on a single GPU. This is a
one-time offline cost, and the resulting artifacts are reused across
training epochs and evaluations.

\section{Conclusion}
We presented \method{}, a graph-relational latent bridge for multi-table QA with frozen LLMs. Beyond feeding relational evidence into a sequential prompt, \method{} constructs a typed graph over rows, column classes, and value groups, applies message passing to expose structural dependencies, and compresses the resulting representation into question-conditioned soft tokens. This design preserves the flexibility of autoregressive language models while adding an explicit inductive bias for joins, repeated values, and cross-row evidence aggregation.
Across single- and multi-table benchmarks, \method{} improves over serialization-only and soft-prompt baselines. The results support the view that tables should be treated as a structured modality. 
Our analysis shows that latent graph tokens act as structural guidance: they help the LLM locate relevant evidence, while arithmetic and fine-grained symbolic computation remain challenging.

Our findings suggest a broader direction for table-language modeling: pre-trained relational encoders that can be reused across datasets, analogous to the role of pre-trained encoders in vision-language models \cite{instructblip}. We see \method{} as a step toward such table foundation interfaces, where relational structure is encoded explicitly and integrated with LLM reasoning without sacrificing the general capabilities of the underlying LLM.


\section{Limitations}

While \method{} is parameter-efficient relative to LLM fine-tuning, it introduces additional preprocessing and graph-encoding overhead. This cost is modest in our setting, but it may become more significant for very large databases, highly connected schemas, or applications requiring low-latency inference. Future work should study scalable table retrieval jointly with graph construction, stronger pretraining for relational encoders, and tighter integration with symbolic tools for exact arithmetic and executable reasoning.

Because \textbf{GRAB} supplements rather than replaces the serialized text, the approach is still bounded by the context window limits of the underlying LLM (e.g., dropping or truncating tables that exceed 8K tokens).
Also, in the graph constructor,  
missing, noisy, or ambiguous schema information may weaken the structural graph and reduce the benefit of message passing. Relatedly, our results are attributable to the pipeline as a whole. We isolate two of its components: question conditioning and encoder depth. The constructor and the message-passing encoder, however, are not separated from each other, since both are required to produce the typed node states the resampler reads.
Furthermore, we have not evaluated zero-shot transfer across datasets with substantially different schema distributions and therefore do not claim cross-dataset schema generalization. Enabling stronger cross-dataset schema transfer in \method{} is an important direction for future work.

In TQA datasets, tests come with only the correct tables required to answer a question. In a more general setting, where databases are given as input, in the current implementation \method{} assumes that a relevant subset of tables has already been retrieved.   
To handle this issue, standard pipelines rely on retrieval modules to first extract a relevant subset of tables \cite{aixel,shen24}. 
A natural next step is to jointly learn retrieval, schema linking, and graph construction so that structural representations remain robust under incomplete or noisy database metadata.


\section{Use of AI assistants}
\label{sec:ai}
When writing this paper, we used ChatGPT to improve the flow of writing and the vocabulary of the initial drafts we manually wrote. Each suggestion has been manually validated by the authors. 


\section*{Acknowledgment}
This work was funded by the French government, through the 3IA Côte d’Azur Investments in the IA-cluster project managed by the National Research
Agency (ANR-23-IACL-0001). 
This project was provided with resources by GENCI at IDRIS, thanks to grants
2025-AD010616649 and 2025-AD010616180.

\bibliography{custom}

\newpage
\appendix

\section{Beyond Table Question Answering}
\label{app:factcheck}
\input{appendix/app_factcheck}

\section{Ablation}
\label{app:full-ablation}
\input{appendix/app_ablation}
\input{appendix/app_additional_exp}

\section{Constructor Details}
\label{app:additional-theory}
\input{appendix/app_theory}

%


\input{appendix/app_logical-bridge2}


\section{Experimental Setup Details}
\label{app:implementation-details}

\subsection{Dataset and Benchmark Details}
\input{appendix/app_dataset}
\label{app:datasets}
    

\subsection{Training Setup}
\input{appendix/app_training}


\section{Structural Hierarchy of TQA Queries}
\label{app:structural-hierarchy}
\input{appendix/app_question-hierarchy}

\input{appendix/app_stress-test}

\section{Architectural Comparison with Frozen-LLM Table Adapters}
\label{app:method-comparison}
Table~\ref{tab:method-comparison} summarizes the main architectural
differences between GRAB and the closest frozen-LLM table adaptation
baselines. The comparison highlights whether each method explicitly
models multi-table structure, conditions its latent representation on
the question, and incorporates foreign-key information.
\begin{table*}[t]
\centering
\small
\begin{tabular}{lccccc}
\toprule
\textbf{Method} & 
\textbf{Multi-table} & 
\textbf{Query-cond.} & 
\textbf{FK-aware} & 
\textbf{Frozen LLM} & 
\textbf{Serialization} \\
\midrule
Prompt tuning & 
\(\times\) & 
\(\times\) & 
\(\times\) & 
\(\checkmark\) & 
\(\checkmark\) \\

TAMO & 
\(\times\) & 
\(\times\) & 
\(\times\) & 
\(\checkmark\) & 
\(\checkmark\) \\

GRAB & 
\(\checkmark\) & 
\(\checkmark\) & 
\(\checkmark\) & 
\(\checkmark\) & 
\(\checkmark\) \\
\bottomrule
\end{tabular}
\caption{Comparison of GRAB with representative frozen-LLM table adaptation methods. GRAB differs by explicitly modeling multi-table relational structure, conditioning its latent structural tokens on the question, and incorporating foreign-key-aware graph construction.}
\label{tab:method-comparison}
\end{table*}
\end{document}

%% file: sec4-short.tex
As illustrated in Figure~\ref{fig:architecture}, instead of relying only on
$S_{\text{\sl table}}$, we define a \textit{graph constructor} $\gamma$ that
lifts the relational context into an explicit heterogeneous graph
$\cG=\gamma(\cT,\cF)$. In this graph, rows, column classes, and value groups are
represented as typed nodes, and foreign-key metadata induces shared column
classes. A graph encoder processes $\cG$ to capture row--column--value
dependencies and cross-table connectivity. To interface with the LLM, we define
a query-conditioned latent bridge that projects the encoded graph into a
fixed-length sequence of $K$ soft tokens, denoted $Z=\{z_1,\dots,z_K\}$. The LLM
then generates the answer from both the textual prompt and our structural prefix.
Throughout this process, the LLM weights remain frozen; only the graph encoder
and bridge are updated.

\stitle{Relational Graph Constructor.}
Intuitively, our constructor translates relational tables into a graph where rows, columns, and actual cell values are all treated as individual nodes. Edges are drawn simply based on inclusion: a row is connected to the values it contains, and a column is connected to the values it can hold. This naturally forces repeated values and foreign-key joins to become shared connection points (hubs) in the graph.

The \textit{graph constructor} $\gamma$ is a fixed, deterministic processing
map. It exposes row--column--value incidence and cross-table join structure
before any neural message passing occurs. We provide more details in~\cref{app:additional-theory}.
Let us consider the foreign-key
metadata in $\mathcal{M}$, that is, 
$
\cF =
\bigl\{\bigl((i,c),(j,d)\bigr) \mid
\text{column $c$ of $T_i$ joins column $d$ of $T_j$}\bigr\}$. The constructor maps the relational input to a tripartite graph $\gamma(\cT,\cF)=\bigl(\cG,\bH^{(0)}\bigr)$
with
$\cG =
\bigl(
\cV_R\cup\cV_C\cup\cV_V,
\cE_{RV}\cup\cE_{CV}
\bigr)$,
where $\cV_R$ are row nodes, $\cV_C$ are column-class nodes, $\cV_V$ are
value-group nodes, $\cE_{RV}$ are row--value incidence edges, and $\cE_{CV}$ are
column--value incidence edges. Furthermore,
$\bH^{(0)}$ denotes the initial row, column and value 
node features $H_R^{(0)}$, $H_C^{(0)}$, and $H_V^{(0)}$.

A \emph{distinguishing characteristic} of our construction is that columns linked by
foreign-key pairs in $\cF$ are represented by a single column-class node. This
embeds joins directly in the graph and makes related tables accessible to later
message passing. For example, in Figure~\ref{fig:architecture}, the city columns in $T_1$ and $T_2$ map to a single node in $\mathcal G$. Nodes and edges are otherwise defined in the natural way:
row--value edges record which value groups occur in each row, and column--value
edges record which value groups belong to each column class. For value nodes, we take values after applying a standardization map.
Here, categorical and textual values are normalized and canonicalized within
their column class, while numerical values are mapped to quantile buckets. 

For initial features, we use a fixed token-level text encoder and extend it to
strings by mean pooling. The graph hidden dimension is inherited from the
embedding model used to initialize the nodes. In our implementation, row and
column nodes are initialized from the same text embedding model, while value
nodes are constructed directly in the same hidden dimension. Hence, all node
types already lie in a common embedding space, and no additional projection is
applied at initialization.

Column nodes are initialized from the header embedding of the corresponding
column occurrence. When a column node represents multiple columns identified by
foreign-key links, we average their header embeddings. Value nodes are not
initialized from the raw cell-value text. Instead, we use a deterministic
embedding of the value-group identifier, augmented with column-type information,
constructed in the graph hidden dimension. Thus, value nodes primarily act as
structural anchors: they indicate where identical categorical/textual values or
discretized numeric buckets occur, while exact values remain available to the
LLM through $S_{\text{\sl table}}$.\looseness=-1

%% file: tables/tab_stress-summary.tex
\begin{table}[t]
\centering
\small
\setlength{\tabcolsep}{4pt}
\begin{tabular}{llccc}
\toprule
 & & Serialized & \textbf{GRAB} & $\Delta$F1 \\
\midrule
\multirow{3}{*}{\textsc{lookup}} & 1 cond.  & 66.59 & 80.72 & $+14.13$ \\
                                 & 2 cond.  & 65.56 & 89.32 & $+23.76$ \\
                                 & 3 cond.  & 73.90 & 95.52 & $+21.62$ \\
\midrule
\multirow{3}{*}{\textsc{count}}  & no cond. &  8.63 & 54.90 & $+46.27$ \\
                                 & 1 cond.  & 23.01 & 54.29 & $+31.28$ \\
                                 & 2 cond.  & 34.69 & 57.14 & $+22.45$ \\
\midrule
\multirow{3}{*}{\textsc{avg}}    & no cond. &  1.23 &  1.23 & $+0.00$  \\
                                 & 1 cond.  &  9.38 & 15.24 & $+5.86$  \\
                                 & 2 cond.  & 28.62 & 40.95 & $+12.33$ \\
\midrule
Group-by \textsc{max} & 1 key   & 42.92 & 72.59 & $+29.67$ \\
Group-by \textsc{avg} & 1 key   & 21.63 & 38.77 & $+17.14$ \\
\bottomrule
\end{tabular}
\caption{Stress-test results (F1) for representative categories. Full table in Appendix~\ref{app:stress-test}.}
\label{tab:stress-summary}
\end{table}

%% file: tables/table-overhead-wide.tex
\begin{table*}[ht!]
\centering
\small
\begin{tabular}{lccc}
\toprule
\textbf{Metric} &
\textbf{Frozen LLM} &
\textbf{GRAB} &
\textbf{Overhead} \\
\midrule
End-to-end inference time
    & 1,101.6 s
    & 830.8 s
    & $-24.6\%$ \\

Decode rate (ms / generated token)
    & 21.4
    & 23.1
    & $+8.0\%$ \\

\quad GNN + resampler + projector
    & --
    & 3.5 ms
    & 1.6\% of query \\

Peak GPU memory
    & 7.98 GB
    & 8.43 GB
    & +0.45 GB ($+5.6\%$) \\

FLOPs / query
    & 6.47 TFLOPs
    & 7.26 TFLOPs
    & $+12.2\%$ \\

Trainable parameters
    & 0
    & 91M
    & 2.29\% of 4.02B \\
\bottomrule
\end{tabular}
\caption{
Computational overhead of GRAB relative to the serialization-only
frozen Qwen3-4B baseline on the StructQA test set.
}
\label{tab:computational_overhead}
\end{table*}

%% file: appendix/app_factcheck.tex
The main focus of our experiments is on table question answering, where the model must extract or reason over tabular content to produce a natural-language answer. We further hypothesize that the representations learned by GRAB are not specific to QA, but can transfer to structurally different tabular tasks. To test this hypothesis, we evaluate on two additional benchmarks: TabFact, a fact verification benchmark in which the model must classify a statement as entailed or refuted by a given table, and Spider, a text-to-SQL benchmark in which the model must generate a structured query from a natural-language question and a database schema. These tasks differ from QA both in output format and in the type of reasoning required, providing a broader test of the encoder's ability to capture table structure in a task-agnostic way. In Table~\ref{tab:grab_tabfact_spider}, we report accuracy for TabFact and execution accuracy for Spider. The results confirm that GRAB is able to extract task-agnostic information that can be used for tasks beyond question answering.

\input{tables/tab_TFC-T2SQL}

%% file: tables/tab_TFC-T2SQL.tex
\begin{table}[h]
\centering
\small
\setlength{\tabcolsep}{4pt}
\begin{tabular}{lccc}
\toprule
\textbf{Dataset} & \textbf{Base} & \begin{tabular}{c}\textbf{Soft}\\\textbf{Prompt}\end{tabular} & \textbf{GRAB} \\
\midrule
TabFact     & 67.86 & 81.15 & \textbf{84.25} \\
Spider & 23.78 & 68.84 & \textbf{72.46} \\
\bottomrule
\end{tabular}
\caption{Comparison of Base, Soft Prompt, and GRAB on TabFact and Spider with Qwen3-4B-Base}
\label{tab:grab_tabfact_spider}
\end{table}

%% file: appendix/app_ablation.tex



\subsection{GNN Architecture and Latent Bridge Ablations}
In the first stage, we fix the GNN at 2 layers and the resampler at 1 head and 8 layers, and sweep over latent count $K \in \{1, 2, 4, 8, 16, 32, 64, 128, 256\}$, with results reported in Table~\ref{tab:latent_sweep}. 

We find that $K{=}32$ and $K{=}64$ perform best while larger values bring no further gain, and very small values ($K \leq 2$) degrade  performance noticeably. Based on this, we carry forward $K \in \{4, 32, 64\}$ as representative small, medium, and large capacity settings.

\input{tables/tab_ablation_latent}

In the second stage, we fix the resampler and jointly sweep GNN depth $L \in \{1, 2, 4, 8, 16\}$ against the three selected latent counts, for a total of 15 configurations. Results are reported in Table~\ref{tab:gnn_latent_joint_ablation}.

\input{tables/tab_ablation_depth_latent}

We further ablate the design of the projection head that bridges the GNN encoder to the LLM input space. Specifically, we compare three variants: a single linear layer (\textbf{Linear}), a two-layer MLP with a non-linearity (\textbf{MLP}), and a deeper MLP with additional hidden layers (\textbf{Deep MLP}). All other components are kept fixed. 

As shown in  Table~\ref{tab:projection_head_ablation}, the linear projection consistently matches or outperforms the non-linear alternatives, suggesting that the resampler already provides sufficient expressivity and that additional capacity in the projection head does not yield further gains, and can even hurt performance, as seen on MMQA and MultiHierTT.

\input{tables/tab_ablation-bridge}

%% file: tables/tab_ablation_latent.tex
\begin{table*}[t]
\centering
\small
\begin{tabular}{lccccc}
\toprule
 & \multicolumn{4}{c}{\textbf{Dataset}} & \\
\cmidrule(lr){2-5}
\textbf{Latents ($K$)} & \textbf{StructQA} & \textbf{HiTab} & \textbf{MultiHierTT} & \textbf{MMQA} & \textbf{Avg} \\
\midrule
1   & 79.87 & 71.84 & 19.35 & 27.03 & 49.52 \\
2   & 79.73 & 71.78 & 19.83 & 26.79 & 49.53 \\
4   & 79.07 & \textbf{75.32} & 20.40 & 28.47 & 50.82 \\
8   & 80.33 & 72.03 & 20.59 & 27.27 & 50.06 \\
16  & 80.60 & 72.92 & 20.11 & 28.95 & 50.65 \\
32  & 80.73 & 73.23 & 18.87 & \textbf{31.34} & \textbf{51.04} \\
64  & 81.27 & 73.17 & \textbf{20.98} & 28.23 & 50.91 \\
128 & 79.53 & 74.12 & 19.92 & 27.99 & 50.39 \\
256 & \textbf{81.47} & 73.11 & 20.31 & 26.56 & 50.36 \\
\bottomrule
\end{tabular}
\caption{Stage 1 ablation: latent count sweep with GNN fixed at 2 layers and resampler fixed at 1 head and 8 layers. Bold denotes the best result per column.}
\label{tab:latent_sweep}
\end{table*}

%% file: tables/tab_ablation_depth_latent.tex
\begin{table}[h]
\centering
\small
\begin{tabular}{lccccc}
\toprule
 & \multicolumn{5}{c}{\textbf{GNN Layers}} \\
\cmidrule(lr){2-6}
\textbf{$K$} & \textbf{1} & \textbf{2} & \textbf{4} & \textbf{8} & \textbf{16} \\
\midrule
4   & 51.40 & 50.82 & 51.17 & 51.64 & 52.19 \\
32  & 52.43 & 51.04 & 51.79 & 52.22 & 52.61 \\
64  & 50.99 & 50.91 & 52.05 & 52.20 & \textbf{52.82} \\
\bottomrule
\end{tabular}
\caption{Joint ablation of GNN depth and latent count $K$ (average over StructQA, HiTab, MultiHierTT, and MMQA). The resampler is fixed at 1 head and 8 layers throughout. Bold denotes the best overall configuration.}
\label{tab:gnn_latent_joint_ablation}
\end{table}

%% file: tables/tab_ablation-bridge.tex
\begin{table}[h]
\centering
\small
\setlength{\tabcolsep}{4pt}
\begin{tabular}{lccc}
\toprule
\textbf{Dataset} &
\textbf{Linear} &
\textbf{MLP} &
\textbf{Deep MLP}\\
\midrule
StructQA & 84.80 & 85.00 & 80.27 \\
HiTab       & 74.49 & 74.05 & 73.55 \\
MultiHierTT & 21.55 & 19.35 & 19.54 \\
MMQA        & 31.10 & 25.84 & 23.68 \\
\bottomrule
\end{tabular}
\caption{Ablation comparison of projection heads across StructProbe, HiTab, MultiHierTT, and MMQA.}
\label{tab:projection_head_ablation}
\end{table}

%% file: appendix/app_additional_exp.tex
\subsection{Impact of Conditioning and Serialization}
To assess the contribution of key components of \method, we conduct two ablation studies. In the first, we remove the question conditioning from the resampler, meaning the GNN produces table embeddings independently of the input question rather than attending to it during the cross-attention aggregation. In the second, we remove the serialized table from the LLM prompt entirely, relying solely on the projected GNN embeddings to convey table content to the language model. These two ablations isolate respectively the role of question-aware table encoding and the role of the text-based table representation as a complementary signal to the graph embeddings. Results are reported in Table~\ref{tab:q_resampler_table_prompt_ablation}.

\input{tables/tab_ablation-QC-Prompt}

\subsection{Robustness}
To verify that the reported results are not an artifact of a particular random initialization, we retrain the best-performing configuration with three different random seeds and report mean and standard deviation across runs. A small variance would confirm that the model converges reliably and that comparisons with baselines are meaningful beyond a single lucky initialization. Results reported in Table~\ref{tab:seed_results}.

\input{tables/tab_abl-seeds}

\subsection{Performance across Different Models}

We use Qwen3-4B as the main backbone because it is open-weight, supports
long-context evaluation, and enables controlled experiments with a frozen
LLM. To assess whether the benefits of \method{} are specific to this
backbone or model scale, we additionally evaluate Gemma~3~4B,
Llama~3.2~3B, Qwen3.5~4B, and Qwen3.5~27B. For each backbone, we compare
\method{} against the corresponding frozen serialization-only model and
a prompt tuning baseline.

\input{tables/tab_different_models_results}

Table~\ref{tab:backbone-generalization} shows that the improvements of
\method{} are not specific to Qwen3-4B. The same pattern appears with
Gemma and Llama backbones as well as with the substantially larger
Qwen3.5-27B model. In all three configurations, \method{} improves over
both the frozen serialization-only and prompt tuning baselines on all four
reported datasets. Importantly, the advantage is retained at 27B
parameters, suggesting that the structural bridge remains useful even
when the underlying LLM is substantially larger. We further evaluate \method{} with Qwen3.5-4B across the complete benchmark
suite to test whether the same behavior persists with the newer Qwen3.5
model family. Results are reported in
Tables~\ref{tab:qwen35-single} and~\ref{tab:qwen35-multi}.

\input{tables/tab_qwen35_full_results}

Across the full benchmark suite, \method{} consistently improves over the
frozen Qwen3.5-4B baseline and generally outperforms prompt tuning.
The gains are particularly pronounced on structure-intensive datasets:
for example, \method{} improves prompt tuning performance from 82.93 to
90.87 on StructProbe, from 30.46 to 38.60 EM on MultiHierTT, and from
34.21 to 41.15 on MMQA.

These results indicate that the effectiveness of the graph-relational
bridge is preserved when moving from Qwen3 to the newer Qwen3.5 family,
rather than being tied to the specific characteristics of the original backbone.

\subsection{Impact of Increasing Context Length}
To evaluate how \method{} behaves as the serialized table content grows,
we test checkpoints trained at the 8K TQA-Bench scale on progressively
larger TQA-Bench contexts. With Qwen3-4B, evaluation is extended to 16K
tokens due to its maximum context length, while Qwen3.5-4B allows us to continue the evaluation up to 128K tokens. In all cases, the checkpoints are trained at the 8K scale and evaluated directly at larger context lengths.

\input{tables/tab_results_tqa_context}

As shown in Table~\ref{tab:context-length-generalization}, \method{}
generalizes beyond the context length used during training. With
Qwen3.5-4B, accuracy decreases gradually as the serialized input grows,
from 92.67 at 16K to 69.57 at 128K. Nevertheless, \method{} remains
substantially above the corresponding frozen serialization-only baseline
at every evaluated scale. The gap is preserved even at 128K, where
\method{} achieves 69.57 accuracy compared with 39.67 for the frozen
LLM.

These results suggest that the structural signal remains useful as the
amount of serialized evidence increases, including at context lengths
well beyond those observed during training. At the same time, the
progressive degradation at 64K and 128K indicates that \method{} does
not eliminate the challenges associated with very long contexts.
Accordingly, these experiments demonstrate robustness to increasing
context length rather than unrestricted scalability to arbitrarily large
tables.

\subsection{Different Serialization Formats}

To assess whether \method{} depends on a particular textual representation
of the tables, we evaluate Markdown and CSV serialization with Qwen3-4B.
We consider two settings. First, we change the serialization format only
at evaluation time, while keeping checkpoints trained on their native
serialization format. This setting measures robustness to an unseen
serialization format. Second, we train \method{} separately using each
serialization format to determine whether the observed sensitivity arises
from the format itself or from a train--test format mismatch.

\input{tables/tab_frozen_serialization_format}

Table~\ref{tab:serialization-format-transfer} shows that \method{} is not
invariant to an unseen serialization format. Changing the format at
evaluation time causes dataset-dependent degradation, particularly on
StructProbe when switching to CSV. Nevertheless, \method{} consistently
outperforms the corresponding frozen LLM under both Markdown and CSV,
suggesting that the structural channel remains beneficial even under
serialization mismatch.

\input{tables/tab_trained_serialization_format}

When the model is instead trained using the same serialization format used
at evaluation time (Table~\ref{tab:serialization-format-training}),
performance substantially recovers across the four datasets. The remaining
differences between Markdown and CSV are relatively modest compared with
the degradation caused by an unseen format change. These results indicate
that \method{} is not serialization-invariant, but its performance is not
tied to a single serialization scheme: the model can adapt effectively to
different textual formats when they are represented during training.

\subsection{Sensitivity to Numerical Bucketing}
\label{app:adaptive-bucketing}

Numerical bucketing is applied independently to each column. In the main
configuration, we use a fixed maximum of 10 buckets. If a numerical column
contains fewer than 10 unique values, no bucketing is applied; otherwise,
its values are partitioned into 10 quantile buckets.

To evaluate sensitivity to this design choice, we additionally implement
an adaptive bucketing strategy in which the number of bins scales with the
cardinality of the numerical column:
\[
n_{\mathrm{bins}}
=
\max\left(
2,
\left\lceil
\sqrt{n_{\mathrm{unique}} \cdot b}
\right\rceil
\right),
\qquad b=10,
\]
where $n_{\mathrm{unique}}$ denotes the number of unique values in the
column.

\input{tables/tab_adaptive_bucketing}

Tables~\ref{tab:adaptive-bucketing-single}
and~\ref{tab:adaptive-bucketing-multi} show that adaptive bucketing yields
performance broadly comparable to the fixed 10-bucket configuration.
Adaptive bucketing improves several single-table benchmarks, including
StructProbe, WikiSQL, HCTQA, and TabMWP. In the multi-table setting, the two
schemes trade off across datasets and metrics without a consistent winner. These results suggest that \method{}'s performance is not strongly tied
to the fixed choice of 10 numerical buckets. Rather, the important design
choice is the discretized structural representation itself, while the exact
number of bins has a comparatively limited and dataset-dependent effect.

%% file: tables/tab_ablation-QC-Prompt.tex
\begin{table}[h]
\centering
\small
\setlength{\tabcolsep}{4pt}
\begin{tabular}{lccc}
\toprule
\textbf{Dataset} &
\begin{tabular}{c}\textbf{No Question}\\\textbf{Conditioning}\end{tabular} &
\begin{tabular}{c}\textbf{No Table}\\\textbf{in Prompt}\end{tabular} &
\begin{tabular}{c}\textbf{Full}\\\textbf{Model}\end{tabular} \\
\midrule
StructQA & 82.07 & 10.07 & 84.80 \\
HiTab       & 74.81 & 4.42  & 74.49 \\
MultiHierTT & 20.21 & 2.87  & 21.55 \\
MMQA        & 29.67 & 8.37  & 31.10 \\
\bottomrule
\end{tabular}
\caption{Ablation comparison across StructProbe, HiTab, MultiHierTT, and MMQA.}
\label{tab:q_resampler_table_prompt_ablation}
\end{table}

%% file: tables/tab_abl-seeds.tex
\begin{table}[h]
\centering
\small
\setlength{\tabcolsep}{3pt}
\begin{tabular}{lcccc}
\toprule
\textbf{Seed} & \textbf{StructQA} & \textbf{HiTab} & \textbf{MultiHierTT} & \textbf{MMQA} \\
\midrule
42 & 84.80 & 74.49 & 21.55 & 31.10 \\
43 & 87.33 & 72.98 & 21.55 & 27.99 \\
44 & 89.66 & 74.12 & 20.88 & 30.38 \\
45 & 87.73 & 74.24 & 21.17 & 27.99 \\
46 & 88.66 & 74.37 & 20.79 & 27.75 \\
\midrule
Mean & 87.64 & 74.04 & 21.19 & 29.04 \\
Std  & 1.82  & 0.62  & 0.34  & 1.56  \\
\bottomrule
\end{tabular}
\caption{Accuracy results across random seeds.}
\label{tab:seed_results}
\end{table}

%% file: tables/tab_different_models_results.tex
\begin{table*}[h]
\centering
\small
\setlength{\tabcolsep}{6pt}
\begin{tabular}{llcccc}
\toprule
\textbf{Backbone}
& \textbf{Method}
& \textbf{StructProbe}
& \textbf{HiTab}
& \textbf{MultiHierTT}
& \textbf{MMQA} \\
&
& \textbf{Acc.}
& \textbf{Acc.}
& \textbf{Acc.}
& \textbf{Acc.} \\
\midrule

\multirow{3}{*}{Gemma 3 4B}
& Frozen LLM  & 14.60 & 17.30 & 3.54 & 5.50 \\
& Prompt Tuning & 52.33 & 22.73 & 4.89 & 8.61 \\
& \textbf{\method{}} & \textbf{71.60} & \textbf{68.94}
& \textbf{10.06} & \textbf{13.88} \\
\midrule

\multirow{3}{*}{Llama 3.2 3B}
& Frozen LLM  & 21.50 & 23.00 & 4.12 & 3.83 \\
& Prompt Tuning & 68.20 & 66.92 & 8.52 & 13.88 \\
& \textbf{\method{}} & \textbf{79.33} & \textbf{68.31}
& \textbf{13.51} & \textbf{25.84} \\
\midrule

\multirow{3}{*}{Qwen3.5 27B}
& Frozen LLM  & 64.47 & 75.00 & 22.61 & 24.88 \\
& Prompt Tuning & 87.33 & 85.29 & 50.48 & 37.32 \\
& \textbf{\method{}} & \textbf{93.13} & \textbf{86.30}
& \textbf{56.80} & \textbf{48.80} \\

\bottomrule
\end{tabular}
\caption{
Generalization of \method{} across LLM families and model scales.
For each backbone, we compare the frozen serialization-only model,
prompt tuning, and \method{} on representative single- and
multi-table benchmarks.
}
\label{tab:backbone-generalization}
\end{table*}

%% file: tables/tab_qwen35_full_results.tex
\begin{table*}[t]
\centering
\small
\setlength{\tabcolsep}{6pt}
\begin{tabular}{lccccccc}
\toprule
\textbf{Method}
& \multicolumn{1}{c}{\textbf{StructProbe}}
& \multicolumn{1}{c}{\textbf{HiTab}}
& \multicolumn{1}{c}{\textbf{WTQ}}
& \multicolumn{1}{c}{\textbf{WikiSQL}}
& \multicolumn{2}{c}{\textbf{HCTQA}}
& \multicolumn{1}{c}{\textbf{TabMWP}} \\
\cmidrule(lr){2-2}
\cmidrule(lr){3-3}
\cmidrule(lr){4-4}
\cmidrule(lr){5-5}
\cmidrule(lr){6-7}
\cmidrule(lr){8-8}
&
\textbf{Acc.}
& \textbf{Acc.}
& \textbf{Acc.}
& \textbf{Acc.}
& \textbf{F1}
& \textbf{CC}
& \textbf{Acc.} \\
\midrule

Frozen LLM
& 31.33
& 58.71
& 43.32
& 59.18
& 33.06
& 16.51
& 48.00 \\

Prompt Tuning
& 82.93
& 78.72
& 66.45
& 90.84
& 82.07
& 62.87
& 88.76 \\

\textbf{GRAB}
& \textbf{90.87}
& \textbf{81.50}
& \textbf{67.27}
& \textbf{91.56}
& \textbf{91.15}
& \textbf{81.31}
& \textbf{92.82} \\

\bottomrule
\end{tabular}

\caption{
Single-table question answering results with Qwen3.5-4B.
\method{} consistently improves over both the frozen serialization-only
baseline and prompt tuning across all reported single-table benchmarks.
}
\label{tab:qwen35-single}
\end{table*}

\begin{table*}[t]
\centering
\small
\setlength{\tabcolsep}{4.5pt}

\resizebox{\textwidth}{!}{%
\begin{tabular}{lccccccccccc}
\toprule
\textbf{Method}
& \multicolumn{1}{c}{\textbf{MultiHierTT}}
& \multicolumn{2}{c}{\textbf{SciTaT}}
& \multicolumn{1}{c}{\textbf{MMQA}}
& \multicolumn{1}{c}{\textbf{TQA-Bench}}
& \multicolumn{2}{c}{\textbf{ATIS}}
& \multicolumn{2}{c}{\textbf{GeoQuery}}
& \multicolumn{2}{c}{\textbf{Spider}} \\
\cmidrule(lr){2-2}
\cmidrule(lr){3-4}
\cmidrule(lr){5-5}
\cmidrule(lr){6-6}
\cmidrule(lr){7-8}
\cmidrule(lr){9-10}
\cmidrule(lr){11-12}
&
\textbf{Acc.}
& \textbf{EM}
& \textbf{F1}
& \textbf{Acc.}
& \textbf{Acc.}
& \textbf{T-EM}
& \textbf{Cell F1}
& \textbf{T-EM}
& \textbf{Cell F1}
& \textbf{T-EM}
& \textbf{Cell F1} \\
\midrule

Frozen LLM
& 5.65
& 10.29
& 29.40
& 5.02
& 41.19
& 1.16
& 20.14
& 0.00
& 9.17
& 3.45
& 2.68 \\

Prompt Tuning
& 30.46
& 30.00
& \textbf{42.62}
& 34.21
& 90.00
& 75.58
& 69.04
& 55.73
& 40.96
& 42.70
& 74.20 \\

\textbf{GRAB}
& \textbf{38.60}
& \textbf{32.94}
& 41.94
& \textbf{41.15}
& \textbf{92.86}
& \textbf{81.40}
& \textbf{74.79}
& \textbf{62.06}
& \textbf{42.32}
& \textbf{44.92}
& \textbf{74.78} \\

\bottomrule
\end{tabular}%
}

\caption{
Multi-table question answering results with Qwen3.5-4B.
\method{} improves over the frozen serialization-only baseline across all
reported metrics and generally outperforms prompt tuning.
}
\label{tab:qwen35-multi}
\end{table*}

%% file: tables/tab_results_tqa_context.tex
\begin{table}[h]
\centering
\small
\setlength{\tabcolsep}{3.5pt}
\resizebox{\columnwidth}{!}{%
\begin{tabular}{lcccc}
\toprule
\textbf{Scale}
& \textbf{GRAB}
& \textbf{Frozen}
& \textbf{GRAB}
& \textbf{Frozen} \\
& \textbf{Qwen3-4B}
& \textbf{Qwen3-4B}
& \textbf{Qwen3.5-4B}
& \textbf{Qwen3.5-4B} \\
\midrule
16K  & 90.48 & 35.95 & \textbf{92.67} & 38.00 \\
32K  & --    & --    & \textbf{90.00} & 36.81 \\
64K  & --    & --    & \textbf{83.90} & 38.81 \\
128K & --    & --    & \textbf{69.57} & 39.67 \\
\bottomrule
\end{tabular}%
}
\caption{
TQA-Bench accuracy under increasing serialized context length.
All \method{} checkpoints are trained at the 8K scale and evaluated
directly on longer contexts. Qwen3-4B is evaluated up to 16K, while
Qwen3.5-4B enables evaluation up to 128K.
}
\label{tab:context-length-generalization}
\end{table}

%% file: tables/tab_frozen_serialization_format.tex
\begin{table}[h]
\centering
\small
\setlength{\tabcolsep}{3.5pt}
\resizebox{\columnwidth}{!}{%
\begin{tabular}{lcccc}
\toprule
\textbf{Dataset}
& \textbf{GRAB (MD)}
& \textbf{Frozen (MD)}
& \textbf{GRAB (CSV)}
& \textbf{Frozen (CSV)} \\
\midrule
StructProbe  & 67.53 & 21.87 & 45.73 & 20.67 \\
HiTab        & 73.74 & 55.05 & 67.99 & 49.75 \\
MultiHierTT  & 14.94 &  7.95 & 13.79 &  8.62 \\
MMQA         & 29.43 &  5.26 & 26.32 &  5.02 \\
\bottomrule
\end{tabular}%
}
\caption{
Robustness to unseen serialization formats. Checkpoints are trained using
their native serialization format and evaluated after changing the table
serialization to Markdown (MD) or CSV.
}
\label{tab:serialization-format-transfer}
\end{table}

%% file: tables/tab_trained_serialization_format.tex
\begin{table}[h]
\centering
\small
\setlength{\tabcolsep}{6pt}
\begin{tabular}{lcc}
\toprule
\textbf{Dataset}
& \textbf{GRAB (MD)}
& \textbf{GRAB (CSV)} \\
\midrule
StructProbe & 87.40 & 85.53 \\
HiTab       & 72.41 & 70.90 \\
MultiHierTT & 18.30 & 17.15 \\
MMQA        & 31.58 & 28.47 \\
\bottomrule
\end{tabular}
\caption{
Performance when \method{} is trained and evaluated  specifically with Markdown or CSV.
}
\label{tab:serialization-format-training}
\end{table}

%% file: tables/tab_adaptive_bucketing.tex
\begin{table*}[t]
\centering
\small
\setlength{\tabcolsep}{6pt}
\begin{tabular}{lccccccc}
\toprule
\textbf{Bucketing}
& \textbf{StructProbe}
& \textbf{HiTab}
& \textbf{WTQ}
& \textbf{WikiSQL}
& \multicolumn{2}{c}{\textbf{HCTQA}}
& \textbf{TabMWP} \\
\cmidrule(lr){2-2}
\cmidrule(lr){3-3}
\cmidrule(lr){4-4}
\cmidrule(lr){5-5}
\cmidrule(lr){6-7}
\cmidrule(lr){8-8}
&
\textbf{Acc.}
& \textbf{Acc.}
& \textbf{Acc.}
& \textbf{Acc.}
& \textbf{F1}
& \textbf{CC}
& \textbf{Acc.} \\
\midrule

10 buckets
& 84.80
& \textbf{74.49}
& 58.33
& 88.30
& 86.38
& 72.81
& 87.01 \\

Adaptive
& \textbf{85.60}
& 74.05
& \textbf{58.85}
& \textbf{89.24}
& \textbf{87.07}
& \textbf{74.21}
& \textbf{88.38} \\

\bottomrule
\end{tabular}

\caption{
Sensitivity of \method{} to the numerical bucketing strategy on
single-table benchmarks. Adaptive bucketing produces broadly comparable
performance to the fixed 10-bucket configuration.
}
\label{tab:adaptive-bucketing-single}
\end{table*}

\begin{table*}[t]
\centering
\small
\setlength{\tabcolsep}{4.5pt}

\resizebox{\textwidth}{!}{%
\begin{tabular}{lccccccccccc}
\toprule
\textbf{Bucketing}
& \textbf{MultiHierTT}
& \multicolumn{2}{c}{\textbf{SciTaT}}
& \textbf{MMQA}
& \textbf{TQA-Bench}
& \multicolumn{2}{c}{\textbf{ATIS}}
& \multicolumn{2}{c}{\textbf{GeoQuery}}
& \multicolumn{2}{c}{\textbf{Spider}} \\
\cmidrule(lr){2-2}
\cmidrule(lr){3-4}
\cmidrule(lr){5-5}
\cmidrule(lr){6-6}
\cmidrule(lr){7-8}
\cmidrule(lr){9-10}
\cmidrule(lr){11-12}
&
\textbf{Acc.}
& \textbf{EM}
& \textbf{F1}
& \textbf{Acc.}
& \textbf{Acc.}
& \textbf{T-EM}
& \textbf{Cell F1}
& \textbf{T-EM}
& \textbf{Cell F1}
& \textbf{T-EM}
& \textbf{Cell F1} \\
\midrule

10 buckets
& \textbf{21.55}
& 26.18
& \textbf{41.38}
& \textbf{31.10}
& \textbf{91.29}
& \textbf{79.07}
& 64.78
& \textbf{51.78}
& \textbf{43.12}
& 36.86
& \textbf{75.11} \\

Adaptive
& 20.69
& \textbf{26.76}
& 40.71
& 29.90
& 91.00
& 76.74
& \textbf{68.14}
& 49.80
& 41.55
& \textbf{38.25}
& 69.10 \\

\bottomrule
\end{tabular}%
}

\caption{
Sensitivity of \method{} to numerical bucketing on multi-table
benchmarks. Adaptive bucketing leads to small gains on some metrics and
small declines on others, with no consistent advantage over the fixed
10-bucket configuration.
}
\label{tab:adaptive-bucketing-multi}
\end{table*}

%% file: appendix/app_theory.tex
For $n\in\bbN$, $n\neq 0$, we let
$[n]=\{1,2,\ldots,n\}$.

The main text already provided some details on the \method{} graph constructor. Here, we give the remaining details in~\cref{subsec:constructor}.
For completeness, in~\cref{subsec:tamo}, we include below the corresponding constructor for the TAMO/HyTrel-style baseline used in our comparisons.

\subsection{\method: Graph Constructor.}
\label{subsec:constructor}
The \textit{graph constructor} $\gamma$ is a fixed, deterministic processing map. It  exposes row--column--value incidence and cross-table join structure before any neural message passing occurs. Let $\cF$ denote the foreign-key metadata in $\mathcal{M}$:
$\cF =
\bigl\{\bigl((i,c),(j,d)\bigr) \mid
\text{column  $c$  of $T_i$ joins
column $d$ of  $T_j$}\bigr\}.$

The constructor maps the relational input to a tripartite graph
$\gamma(\cT,\cF) =
\bigl(\cG,\bH^{(0)}\bigr)$, with $
\cG =
\bigl(
\cV_R\cup\cV_C\cup\cV_V,
\cE_{RV}\cup\cE_{CV}
\bigr)$,
where $\cV_R$ are row nodes, $\cV_C$ are column-class nodes, $\cV_V$ are value-group nodes, $\cE_{RV}$ are row--value incidence edges, and $\cE_{CV}$ are column--value incidence edges. Furthermore,
$\bH^{(0)}=\{\bh^{(0)}_v\}_{v\in\cV_R\cup\cV_C\cup\cV_V}$ are the initial node features.  In more detail:

\textit{Column classes.}
Let
$\cC_{\mathrm{occ}}
=
\{(i,c) \mid i\in[n],\, c\in C_i\}$
be the set of column occurrences across all tables. The relation $\sim_{\cF}$ is the smallest equivalence relation on $\cC_{\mathrm{occ}}$ containing all foreign-key pairs in $\cF$. We write $[(i,c)]_{\cF}$ for the equivalence class of column occurrence $(i,c)$, and use $\alpha$ for a generic class. When $\cF=\emptyset$, the relation is the identity and no columns are merged. We denote the set of column classes by
$\cC_{\cF} = \cC_{\mathrm{occ}}/\!\sim_{\cF}$. For example, in Figure~\ref{fig:architecture}, the city columns in $T_1$ and $T_2$ map to a single node in $\mathcal G$.
Each class $\alpha$ inherits a single type $\tau_\alpha\in\{\mathrm{cat},\mathrm{num},\mathrm{text}\}$. 

\textit{Nodes.}
Row nodes are indexed globally across tables, column nodes correspond to FK-induced column classes, and value nodes correspond to canonical value groups within a column class:
$
\cV_R
=
\bigl\{\rho_{i,r} \mid i\in[n],\, r\in R_i\bigr\}$,
$\cV_C
=
\bigl\{c_\alpha \mid \alpha \in \cC_{\cF}\bigr\}$, and
$\cV_V
=
\bigl\{v_{\alpha,g}
\mid
\alpha \in \cC_{\cF},\,
g\in\mathrm{Im}(g_\alpha)
\bigr\}
$.
Here, for each column class $\alpha$, the grouping map $g_\alpha$ standardizes cell contents:
\[
g_\alpha(x)
=
\begin{cases}
\mathrm{id}_\alpha\bigl(\mathrm{norm}(x)\bigr),
& \tau_\alpha\in\{\mathrm{cat},\mathrm{text}\},\\[2pt]
q_\alpha(x),
& \tau_\alpha=\mathrm{num}.
\end{cases}
\]
where $\mathrm{norm}(\cdot)$ normalizes the cell string and $\mathrm{id}_\alpha(\cdot)$ assigns a unique identifier to each distinct normalized value within $\alpha$, while 
$q_\alpha(\cdot)$ maps values to quantile-based buckets fitted on all numeric values in the class. A value node $v_{\alpha,g}$ therefore represents a repeated categorical/textual value or a numeric range, not an individual cell occurrence.\footnote{Value groups are not replacements for exact numeric content, which remains available in the textual prompt.}

\textit{Edges.} 
For each cell $x^{(i)}_{r,c}$ (value in column $c$ in row $r$ in $R_i$),
let $\alpha=[(i,c)]_{\cF}$ and $g=g_\alpha(x^{(i)}_{r,c})$. The constructor adds the row--value edge
$\{\rho_{i,r},v_{\alpha,g}\}$ to $\cE_{RV}$ and the column--value edge
$\{c_\alpha,v_{\alpha,g}\}$ to $\cE_{CV}$. Thus, repeated values become shared graph neighborhoods, and foreign-key joins become explicit connectivity through shared column classes and value groups. 

\textit{Initial features.}
Let $\phi_{\mathrm{tok}}$ be a fixed token-level text encoder. For any string
$s\in\Sigma^*$, define 
$\phi(s)$ as $
\mathrm{MeanPool}\bigl(\phi_{\mathrm{tok}}(s)\bigr)$.

The graph's hidden feature dimension is inherited from the embedding model used to
initialize the nodes. In our implementation, row and column nodes are initialized
from the same text embedding model, while value nodes are constructed directly
in the same hidden dimension. Therefore, all node types already lie in a common  
embedding space and no additional projection is applied at initialization. We
write $\Vert$ for string concatenation.

For a row node $\rho_{i,r}$, we encode a row-level string formed by
concatenating header--value pairs:
\[
s_{i,r}
=
\big\Vert_{c\in C_i}
\bigl[\,h^{(i)}_c : x^{(i)}_{r,c}\,\bigr],
\qquad
\bh^{(0)}_{\rho_{i,r}}
=
\phi(s_{i,r}).
\]

For a column node $c_\alpha$, we initialize from the header embedding of the
corresponding column occurrence. If $\alpha$ contains multiple
FK-linked column occurrences, we average their header embeddings
$
\bh^{(0)}_{c_\alpha}
=
\frac{1}{|\alpha|}
\sum_{(i,c)\in\alpha}
\phi\bigl(h^{(i)}_c\bigr)$.

For a value node $v_{\alpha,g}$, we do not initialize from the cell-value text.
Instead, we use a deterministic embedding of the value-group identifier,
augmented with the column-type information, constructed in the graph
hidden dimension. Thus, value nodes primarily act as structural anchors: they
indicate where identical categorical/textual values or discretized numeric
buckets occur, while exact values are available to the LLM with
$S_{\text{\sl table}}$.
\newcommand{\TAMO}{\textbf{TAMO}}
\subsection{\TAMO: Graph Constructor}
\label{subsec:tamo}
We describe the TAMO/HyTrel structural encoder for a single flat table. Let
$T$ be a table with row set $R$, column set $C$, and cell values
$x_{r,c}$ in $V$, for $r\in R$ and $c\in C$. Each column carries a header
string $h_c$ and a declared type
$\tau_c\in\{\mathrm{cat},\mathrm{num},\mathrm{text}\}$.

TAMO/HyTrel constructs a hypergraph in which cell occurrences are primitive
nodes, while rows, columns, and the whole table are represented as hyperedges.
Equivalently, we represent this hypergraph by its typed incidence graph. The
constructor returns
$
\gamma_{\mathrm{HT}}(T)
=
\bigl(\cG_{\mathrm{HT}},\bH^{(0)}_{\mathrm{HT}}\bigr)$,
with 
\[
\cG_{\mathrm{HT}}
=
\bigl(
\cV_X\cup\cV_R\cup\cV_C\cup\cV_T,
\cE_{XR}\cup\cE_{XC}\cup\cE_{XT}
\bigr),
\]
where $\cV_X$ are cell-occurrence nodes, $\cV_R$ are row-hyperedge nodes,
$\cV_C$ are column-hyperedge nodes, $\cV_T$ contains the table-hyperedge node,
and $\cE_{XR}$, $\cE_{XC}$, and $\cE_{XT}$ encode incidences between cell nodes
and row, column, and table hyperedge nodes. The initial node features are
\[
\bH^{(0)}_{\mathrm{HT}}
=
\{\bh^{(0)}_v\}_{v\in\cV_X\cup\cV_R\cup\cV_C\cup\cV_T}.
\]

\paragraph{Cell, row-hyperedge, column-hyperedge, and table-hyperedge nodes.}
For every cell occurrence $x_{r,c}$ we introduce a cell node; for every row we
introduce a row-hyperedge node; for every column we introduce a
column-hyperedge node; and for the whole table we introduce a table-hyperedge
node:
\begin{align*}
\cV_X
&=
\bigl\{u_{r,c} \mid r\in R,\ c\in C\bigr\},\\
\cV_R
&=
\bigl\{\rho_r \mid r\in R\bigr\},\\
\cV_C
&=
\bigl\{\kappa_c \mid c\in C\bigr\},\\
\cV_T
&=
\bigl\{\theta_T\bigr\}.
\end{align*}
Here $u_{r,c}$ denotes the node corresponding to the cell occurrence whose
surface value is $x_{r,c}$. Thus, equal cell values in different positions still
give distinct cell nodes. The node $\theta_T$ represents the whole-table
hyperedge and provides a global aggregation channel.

\paragraph{Incidence edges.}
The incidence edges encode membership of cells in rows, columns, and the whole
table. For each cell occurrence $u_{r,c}$, we add
$(u_{r,c},\rho_r)$ to $\cE_{XR}$, $(u_{r,c},\kappa_c)$ to $\cE_{XC}$, and
$(u_{r,c},\theta_T)$ to $\cE_{XT}$:
\begin{align*}
\cE_{XR}
&=
\bigl\{
(u_{r,c},\rho_r)
\mid
r\in R,\ c\in C
\bigr\},\\
\cE_{XC}
&=
\bigl\{
(u_{r,c},\kappa_c)
\mid
r\in R,\ c\in C
\bigr\},\\
\cE_{XT}
&=
\bigl\{
(u_{r,c},\theta_T)
\mid
r\in R,\ c\in C
\bigr\}.
\end{align*}
Equivalently, each cell node belongs to exactly one row hyperedge, one column
hyperedge, and the table hyperedge.

\paragraph{Initial features.}
Let $\phi_{\mathrm{tok}}$ be a fixed token-level text encoder. For any string
$s\in\Sigma^*$, define its pooled representation by
\[
\phi(s)
=
\mathrm{MeanPool}\bigl(\phi_{\mathrm{tok}}(s)\bigr).
\]
For a cell node $u_{r,c}$, the cell-value text is used:
\[
s^X_{r,c}
=
x_{r,c},
\qquad
\bh^{(0)}_{u_{r,c}}
=
\phi(s^X_{r,c}).
\]
Thus, the value of the cell is encoded as the content of a cell-occurrence node.
For a column-hyperedge node $\kappa_c$, one initializes from the column header:
\[
\bh^{(0)}_{\kappa_c}
=
\phi(h_c).
\]
For a row-hyperedge node $\rho_r$, a learned or random initialization is used:
\[
\bh^{(0)}_{\rho_r}
=
\bu^R_r,
\]
where $\bu^R_r\in\mathbb{R}^{d_h}$ is an initialized row-hyperedge embedding.
For the table-hyperedge node $\theta_T$, one initializes from a table caption,
title, or identifier when available:
\[
\bh^{(0)}_{\theta_T}
=
\phi(s^T),
\]
where $s^T$ denotes the available textual description of the table. If no such
description is available, a learned or random table-hyperedge initialization can
be used instead:
\[
\bh^{(0)}_{\theta_T}
=
\bu^T,
\]
where $\bu^T\in\mathbb{R}^{d_h}$ is an initialized table-hyperedge embedding.

%% file: appendix/app_logical-bridge2.tex
\section{Formal View of the Latent Bridge}
\label{app:logical-bridge}

We next formalize the discussion in \cref{sec:theory} about the role and
limits of the latent bridge. As we will see, the bridge does not increase the
expressive power of the message-passing encoder. Rather, it provides a finite,
question-conditioned readout over the graph features already made available by
the constructor and encoder.

\newcommand{\GML}{\mathrm{GML}}
\newcommand{\QRead}{\mathrm{QRead}}
\newcommand{\cQ}{\mathcal Q}
\newcommand{\cA}{\mathcal A}

The relevant starting point is that the theoretical limits of message-passing
GNNs are well studied. Their expressive power is closely related to the
Weisfeiler--Leman hierarchy and to corresponding fragments of first-order logic
\cite{gilmer17,Mor+2019,Xu+2018b,barcelo20}. More recently, the
\textit{Grables} framework \cite{grables} extended this perspective to
tabular learning, showing that row-local methods fail on
``extension-sensitive'' queries: tasks whose answers depend on cross-row
structure, such as counting, overlaps, or joins. This motivates making the
graph construction explicit, as we do here, so that these relations are exposed
to message passing rather than left implicit in a serialized table.


\subsection{Logical View of the Constructed Graph}
We here use the connection between message-passing and graded modal logic \cite{barcelo20}.
Our graph constructor maps the relational input to the tripartite structure
\[
    \cG
    =
    \gamma(\cT,\cF)
    =
    \bigl(
    \cV_R\cup\cV_C\cup\cV_V,
    \cE_{RV}\cup\cE_{CV}
    \bigr),
\]
where row nodes are denoted by $\rho_{i,r}$ (with $i$ indicating the source table), column-class nodes by
$c_\alpha$, and value-group nodes by $v_{\alpha,g}$. We view $\cG$ as a finite
relational structure with unary predicates
\[
    \mathrm{Row}(x),\qquad
    \mathrm{Col}(x),\qquad
    \mathrm{Val}(x),
\]
type-specific refinements such as $\mathrm{Row}_i(x)$ (identifying rows belonging to table $i$) and $\mathrm{Val}_\alpha(x)$ (identifying values in column class $\alpha$), and
binary incidence relations $\cE_{RV}(r,v)$ and $\cE_{CV}(c,v)$.

Let $\GML_L$ denote graded modal logic of modal depth at most $L$ over this
tripartite vocabulary. Its characteristic modality is
\[
    \exists^{\geq N}y\,
    \bigl(\cE_\rho(y,x)\wedge \varphi(y)\bigr),
\]
where $\cE_\rho$ ranges over the typed incidence relations and their inverses.
This modality identifies nodes $x$ that have at least $N$ neighbors $y$
satisfying $\varphi$. Under standard expressivity assumptions on typed
multiset aggregation, an $L$-layer message-passing encoder can represent
depth-$L$ GML facts: one message-passing layer corresponds to one step of
graded neighborhood inspection. Thus, the graph encoder is best understood as
a local logical feature extractor over the tripartite table graph.

The constructor matters because it makes useful table relations local. For
example, repeated values in a column class are detected by the value-node
formula
\[
    \mathrm{Dup}_\alpha(v)
    :=
    \mathrm{Val}_\alpha(v)
    \wedge
    \exists^{\geq 2} r\,
    \bigl(
    \cE_{RV}(r,v)\wedge \mathrm{Row}(r)
    \bigr).
\]
Hence duplicate and equality information is no longer merely a coincidence
between cell strings in a serialization; it is a one-hop counting fact around
a value node. Similarly, if two foreign-key-linked columns are represented by
the same column class $\alpha$, then rows joined by that key are connected
through a shared value node. A bounded-hop join therefore becomes a
bounded-depth GML pattern in $\cG$.

\subsection{The Bridge as a Question-Conditioned Readout}

After message passing, the bridge receives three typed multisets of node
states: rows, columns, and values. The question-conditioned resampler selects
and compresses information from these states into $K$ latent tokens. A useful
logical abstraction is therefore a typed, question-conditioned readout over
GML-definable node properties.

For a formula $\varphi$ and a node type $X\in\{R,C,V\}$, write
\[
    \#_X\varphi(\cG)
    =
    \bigl|
    \{x\in\cV_X \mid \cG,x\models \varphi(x)\}
    \bigr|,
\]
for the number of nodes of type $X$ satisfying $\varphi$. The bridge can be
idealized as a finite sketch
\[
\begin{aligned}
    s_Q(\cG)=\eta_Q\bigl(&
        \#_R\varphi^R_{Q,1}(\cG),\ldots,
        \#_R\varphi^R_{Q,m_R}(\cG),\\
        &\#_C\varphi^C_{Q,1}(\cG),\ldots,
        \#_C\varphi^C_{Q,m_C}(\cG),\\
        &\#_V\varphi^V_{Q,1}(\cG),\ldots,
        \#_V\varphi^V_{Q,m_V}(\cG)
    \bigr),
\end{aligned}
\]
where the formulas are depth-$L$ GML formulas and the finite map $\eta_Q$ may
depend on the question. We denote this abstraction by
\[
    \QRead_B(\GML_L).
\]

This abstraction should be read conservatively. The bridge does not create new
message-passing information; it selects and compresses what the encoder has
already made available. If two graphs have the same typed multisets of
depth-$L$ GML node types, then any permutation-invariant bridge reading only
those encoded states receives the same graph-side information. Question
conditioning at the bridge can improve relevance and compression, but it
cannot recover distinctions erased by the constructor or encoder.

\subsection{Why Question Conditioning Helps}

The bridge has fixed capacity, so the relevant issue is not only which facts
are locally encodable, but also how many graph states the bridge must
separate. Let $\mathfrak G$ be a finite set of possible constructed
tripartite graphs. Each question $Q\in\cQ$ induces an answer map
\[
    a_Q:\mathfrak G\to \cA_Q
\]
and therefore a partition $\Pi_Q$ of $\mathfrak G$:
\[
    \cG\equiv_Q \cG'
    \quad\Longleftrightarrow\quad
    a_Q(\cG)=a_Q(\cG').
\]
Let
\[
    \Pi_{\cQ}
    =
    \bigwedge_{Q\in\cQ}\Pi_Q
\]
be the common refinement. Thus, two graphs are equivalent under $\Pi_{\cQ}$
exactly when they have the same answer to every question in $\cQ$.

We call a bridge \emph{exact} for a question family if its code is sufficient
to recover the correct answer for every graph in $\mathfrak G$ and every
question in the family. In the question-agnostic case, the bridge uses one
code map for all questions. In the question-conditioned case, the bridge may
use a different code map for each $Q$.

\begin{proposition}[Exact sketch complexity]
\label{prop:exact-sketch-complexity}
For a finite graph class $\mathfrak G$ and finite question family $\cQ$, the
minimum number of bits required by a question-agnostic exact bridge is
\[
    C_{\mathrm{agn}}(\cQ)
    =
    \left\lceil
    \log_2 |\Pi_{\cQ}|
    \right\rceil .
\]
The minimum number of bits required by a question-conditioned exact bridge is
\[
    C_{\mathrm{cond}}(\cQ)
    =
    \left\lceil
    \log_2
    \max_{Q\in\cQ}
    |\Pi_Q|
    \right\rceil .
\]
Consequently,
\[
    C_{\mathrm{cond}}(\cQ)
    \leq
    C_{\mathrm{agn}}(\cQ),
\]
and the inequality can be strict.
\end{proposition}

\begin{proof}
A question-agnostic sketch $s:\mathfrak G\to\{0,1\}^B$ induces a partition
$\Pi_s$ of graph states. If $s$ is exact for all $Q\in\cQ$, then two graphs
with the same sketch must have the same answer to every question. Hence
$\Pi_s$ refines $\Pi_{\cQ}$, so $s$ must use at least $|\Pi_{\cQ}|$ distinct
codes. Thus $2^B\geq |\Pi_{\cQ}|$. This gives the lower bound, and it is tight
by assigning one code to each block of $\Pi_{\cQ}$.

For the question-conditioned case, the same argument applies separately to
each $Q$. Exactness for $Q$ requires at least $|\Pi_Q|$ codes, and this is
tight by encoding the block of $\Pi_Q$. Since the same bit budget must work
for every question, the required number of bits is
\[
    \left\lceil
    \log_2
    \max_{Q\in\cQ}
    |\Pi_Q|
    \right\rceil .
\]
Finally, $\Pi_{\cQ}$ refines every $\Pi_Q$, so
$|\Pi_Q|\leq |\Pi_{\cQ}|$ for all $Q\in\cQ$.
\end{proof}

\paragraph{Exponential separation in code space.}
The gap can be exponential at the level of distinguishable graph classes. Let
there be $m$ independent binary facts
$b_1(\cG),\ldots,b_m(\cG)\in\{0,1\}$ about the constructed graph, and let
$\cQ=\{Q_1,\ldots,Q_m\}$ where
\[
    a_{Q_j}(\cG)=b_j(\cG).
\]
Assume every bit vector $b\in\{0,1\}^m$ is realized by some graph in
$\mathfrak G$. Then, for each fixed question $Q_j$, the partition
$\Pi_{Q_j}$ has two blocks, so
\[
    C_{\mathrm{cond}}(\cQ)=1.
\]
However, the joint answer vector
\[
    (a_{Q_1}(\cG),\ldots,a_{Q_m}(\cG))
    =
    (b_1(\cG),\ldots,b_m(\cG))
\]
can take all $2^m$ values. Hence
\[
    |\Pi_{\cQ}|=2^m,
    \qquad
    C_{\mathrm{agn}}(\cQ)=m.
\]
Thus, the question-conditioned bridge needs to distinguish only two answer
classes for the current question, while a question-agnostic bridge must
distinguish $2^m$ joint classes. In bits, this is a gap of $m$ versus $1$; in
codewords, it is exponential.

%% file: appendix/app_dataset.tex
Table~\ref{tab:dataset_stats} summarises the number of samples per split and the
fraction excluded by our token-budget filter. To keep all inputs within a context window of the LLM backbone, we discard any sample whose linearised table exceeds
\textbf{8{,}192 tokens}, as measured by the Qwen3-4B tokenizer. The resulting indices are stored in a skip file and applied consistently to training, validation, and test splits, as well as to graph pre-computation, ensuring that
excluded samples never appear in any evaluation. This is applied to all the baselines when training and testing. The vast majority of datasets are unaffected ($0\%$ skipped); the most impacted datasets are MMQA ($\approx13$--$14\%$ per split, owing to long concatenated table texts) and Spider ($\approx9\%$ train, $27\%$ test). Most datasets are well within the 8{,}192-token budget at the median,
confirming that filtering has a minor impact. The high maximum for WTQ (25{,}943 tokens) and HCTQA (7{,}705 tokens) reflects a small number of extremely wide
Wikipedia tables. For the dataset which did not provide a split, the full dataset has been split into three parts. 

\begin{table*}[t]
\centering
\small
\setlength{\tabcolsep}{4pt}
\begin{tabular}{llcccccc}
\toprule
\textbf{Dataset} & \textbf{Task} &
\textbf{Train} & \textbf{Val} & \textbf{Test} &
\begin{tabular}{c}\textbf{Skip}\\\textbf{train}\end{tabular} &
\begin{tabular}{c}\textbf{Skip}\\\textbf{val}\end{tabular} &
\begin{tabular}{c}\textbf{Skip}\\\textbf{test}\end{tabular} \\
\midrule
StructQA & Table QA          & 4{,}500  & 1{,}500  & 1{,}500   & 0.00  & 0.00  & 0.00  \\
HiTab       & Hier.\ table QA   & 7{,}417  & 1{,}671  & 1{,}584   & 0.00  & 0.00  & 0.00  \\
WTQ         & Table QA          & 11{,}321 & 2{,}831  & 4{,}344   & 1.25  & 0.64  & 1.27  \\
WikiSQL     & TableQA       & 56{,}355 & 8{,}421  & 15{,}878  & 0.27  & 0.20  & 0.40  \\
HCTQA       & Table QA          & 62{,}053 & 7{,}740  & 7{,}789   & 0.03  & 0.19  & 0.00  \\
TabMWP      & Math on tables    & 23{,}059 & 7{,}686  & 7{,}686   & 0.00  & 0.00  & 0.00  \\
TabFact     & Fact verif.       & 92{,}283 & 12{,}792 & 12{,}779  & 0.00  & 0.00  & 0.00  \\
MultiHierTT & Multi-table QA    & 7{,}047  & 783      & 1{,}044   & 0.00  & 0.00  & 0.00  \\
SciTaT      & Sci.\ table QA    & 11{,}573 & 1{,}282  & 953       & 0.00  & 0.00  & 0.00  \\
MMQA        & Multi-table QA    & 2{,}292  & 515      & 488       & 12.91 & 12.82 & 14.34 \\
TQA-Bench   & Multi-table QA    & 9{,}800  & 2{,}100  & 2{,}100   & 0.00  & 0.00  & 0.00  \\
Spider SQL  & Multi-table QA & 6{,}044  & 671      & 985$^{*}$ & 9.00  & 10.13 & 27.01 \\
ATIS        & Multi-table QA        & 384      & 45       & 86        & 0.00  & 0.00  & 0.00  \\
GeoQuery    & Multi-table QA      & 530      & 49       & 253       & 0.00 & 0.00 & 0.00 \\
\bottomrule
\end{tabular}
\caption{Dataset split sizes and percentage of samples removed by the 8{,}192-token table-length filter. Counts reflect the dataset size before filtering. *Since Spider has no test set available, the validation set has been used for testing only purposes.}
\label{tab:dataset_stats}
\end{table*}

\subsection{Per-Row and Column-Header Token Statistics}
\label{app:datasets:rowcol}

For graph-based encoder variants, each table row is tokenised independently in the format
\texttt{col\textsubscript{n}: val\textsubscript{n}},
and each column header is tokenised separately. Table~\ref{tab:rowcol_stats} reports the
maximum observed lengths using the Qwen3-Embedding-0.6B tokenizer across all splits, which determine the safe settings for the max limit for rows and columns length. 

\begin{table}[t]
\centering
\small
\setlength{\tabcolsep}{5pt}

\begin{tabular}{lcc}
\toprule
\textbf{Dataset} &
\begin{tabular}{c}\textbf{Row}\\\textbf{max}\end{tabular} &
\begin{tabular}{c}\textbf{Header}\\\textbf{max}\end{tabular} \\
\midrule
StructQA & 959     & 20  \\
HiTab       & 2{,}206 & 214 \\
WTQ         & 994     & 40  \\
WikiSQL     & 507     & 175 \\
HCTQA       & 1{,}023 & 38  \\
TabMWP      & 97      & 11  \\
MultiHierTT & 691     & 117 \\
SciTaT      & 1{,}545 & 296 \\
TQA-Bench   & 337     & 7   \\
Spider SQL  & 342     & 9   \\
GeoQuery    & 52      & 4   \\
ATIS        & 122     & 5   \\
\bottomrule
\end{tabular}
\caption{Maximum row and column-header token lengths per dataset using the
Qwen3-Embedding-0.6B tokenizer across all splits.}
\label{tab:rowcol_stats}
\end{table}

%% file: appendix/app_training.tex
We use a unified training setup across all experiments to ensure that differences in performance are attributable to the model components rather than optimization choices. Unless otherwise stated, the LLM backbone is kept frozen and only the task-specific adaptation modules are trained. 

All models are optimized with AdamW ($\beta_1{=}0.9$, $\beta_2{=}0.95$, weight decay $0.05$) at a learning rate of $10^{-4}$,
  following a half-cycle cosine decay schedule with 1 epoch of linear warmup and a minimum learning rate of $5{\times}10^{-6}$.
  Gradient norms are clipped to $0.1$. Training runs for up to 10 epochs with early stopping (patience 3), using an effective batch size of 32, adapted on the number of GPUs.
 The effective training time depends on the number of GPUs used for training, which linearly reduces the time per epoch through data parallelism, but the per-step cost remains dominated by the frozen LLM's forward pass, which cannot be
  skipped since its intermediate activations are required for the gradient signal to flow back through the projector and table encoder. Consequently, total time scales with dataset size: a small dataset can complete in under two hours (StructQA), while a large one may require close to a full day (HCT-QA), even with the same hardware configuration.
  All experiments use seed 42.

    \paragraph{Base LLM and LoRA.}
  When LoRA is applied, we use rank $r{=}8$, $\alpha{=}16$, dropout $0.05$, targeting the query and value projection matrices (\texttt{q\_proj}, \texttt{v\_proj}), with no bias adaptation. In all GNN encoder runs the LLM is otherwise frozen, leaving only the table encoder and projector trainable.

  \paragraph{Soft prompt baseline.}
  The soft prompt model prepends 10 learnable virtual tokens to the LLM input with a dimension of 1024. 

  \paragraph{TableLlama.}
TableLlama~\cite{tablellama} is fine-tuned with LoRA using rank $r=16$, scaling factor $\alpha=32$, and dropout $0.05$. LoRA adapters are applied to all attention projection matrices:
\texttt{q\_proj}, \texttt{k\_proj}, \texttt{v\_proj}, and
\texttt{o\_proj}. The model is trained with learning rate
$2\times10^{-4}$, weight decay $0.01$, and a linear warmup over the first $6\%$ of training steps. LLaMA-2's absolute position embeddings are capped at 4096, so the effective input length is necessarily clamped to 4,096 tokens in all runs. Within this limit, only the table is truncated, while the instruction prefix and question are always preserved. 

\paragraph{MultiTabQA.}
MultiTabQA~\cite{pal-etal-2023-multitabqa} is implemented using the base model checkpoint fine-tuned with learning rate $10^{-4}$, weight decay $0.01$, and a
linear warmup over the first $6\%$ of training steps. For Atis, GeoQuery and Spider the checkpoints were already available. Its absolute
position embeddings are capped at $1{,}024$ tokens due to architectural constraints. As a result, the effective input
length is limited to $1{,}024$ tokens, regardless of the configured \texttt{max\_length}. This substantially limits the amount of table content that the model can process and makes it less suitable for large inputs.
  
  \paragraph{GNN encoder.}
  The graph encoder operates on a tripartite graph of row ($R$), column ($C$), and value ($V$) nodes. Node embeddings are
  precomputed offline using Qwen3-Embedding-0.6B and stored as \texttt{float16} tensors, together with question token
  embeddings, row/column validity masks, and adjacency matrices encoding $R$--$V$ and $C$--$V$ edges. The GNN applies 1 message-passing layer over this structure, producing 32 row latents, 32 column latents, and 32 value latents. A cross-attention resampler with 4 heads and 2 layers then pools these into a fixed-size representation per table. GNN dropout is $0.1$; A single linear projector maps the encoder output to the LLM embedding space, initialized with Xavier uniform (gain $0.01$).

%% file: appendix/app_question-hierarchy.tex
This appendix introduces a taxonomy of table question answering queries along three orthogonal axes: the form of the expected answer, the structural depth of table access required to locate evidence, and the computational path required to derive the final value. The taxonomy supports the stress-test in Appendix~\ref{app:stress-test}, which varies the structural and computational axes independently to attribute \method{}'s gains (and limitations) to a specific source. The answer-type axis is included for completeness and to justify why we hold it fixed in the stress-test, as we discuss below.

\subsection{Axis I: Answer Type}
\label{app:answer-type}

The first axis concerns the \emph{form} of the expected answer, independently of how it is computed. We distinguish three classes.

\paragraph{Boolean queries (A1).} The answer is a truth value (yes/no), e.g., \emph{``Is the value in column $c$ greater than $k$?''} or \emph{``Do any two rows share the same value in column $c$?''}

\paragraph{Retrieval queries (A2).} The answer is a value that exists verbatim in the table and is identified by locating the right position, e.g., \emph{``What is the value of column $c$ for the row where column $d$ equals $a$?''}

\paragraph{Derived queries (A3).} The answer is computed from the table and does not need to appear in it, e.g., \emph{``What is the average of column $c$ for rows where column $d$ equals $a$?''}

\paragraph{Why we do not vary this axis.} Answer type governs the evaluation protocol but not the reasoning required to produce a correct answer. A boolean question such as \emph{``Is the average salary of employees in London higher than the company-wide average?''} is A1 by output form but requires two aggregations and a comparison, making it as demanding as any A3 query. The only effect that is specific to A1 is that the binary output space inflates baseline accuracy through guessing, masking rather than revealing reasoning difficulty. We therefore fix the answer type at A2/A3 in the stress-test and vary the two axes that actually expose hardness: structural depth and computational path.

\subsection{Axis II: Structural Depth}
\label{app:structural-depth}

The second axis concerns how much relational structure must be accessed to locate the evidence required for the answer. We distinguish four levels, ordered by minimum structural access. Each question is assigned to the lowest sufficient level.

\paragraph{Scan-sufficient queries (S1).} The answer is determined by locating the right row and reading off a cell value; no aggregation across rows is needed. Example: \emph{``What is the value in column $c$ for the row where column $d$ equals $a$?''} A model reading the serialized table sequentially has access to all required information.

\paragraph{Single-column queries (S2).} The answer requires aggregating or filtering over the values of a single column, e.g., \emph{``How many rows have value $a$ in column $c$?''} or \emph{``What is the average of column $c$ for rows where column $d$ equals $a$?''} The flat serialization contains all necessary values but does not make cross-row statistics explicit.

\paragraph{Multi-column queries (S3).} The answer requires reasoning jointly over two or more columns within a single table, e.g., \emph{``Which value in column $c_1$ co-occurs most often with value $a$ in column $c_2$?''} or \emph{``Is there a pair of rows sharing values in both $c_1$ and $c_2$?''} These cannot be decomposed into independent single-column computations; the joint distribution across columns must be accessible to the model.

\paragraph{Multi-table queries (S4).} The answer requires traversing two or more tables via foreign-key relationships, e.g., \emph{``What is the total of column $c$ in $T_2$ for all rows linked to this row in $T_1$ via column $d$?''} The relevant evidence is distributed across tables and cannot be recovered from any single table in isolation.

\subsection{Axis III: Computational Path}
\label{app:computational-path}

The third axis concerns the arithmetic operations performed \emph{after} the relevant data has been located. This axis is orthogonal to both answer type and structural depth: a scan-sufficient boolean query may require multi-step arithmetic, while a multi-table retrieval query may require no computation beyond identification. Separating this axis is essential for diagnosis: a model that correctly identifies the relevant rows but produces a wrong numeric answer is failing on computation, not on structural reasoning, and the two failure modes call for different interventions.

\paragraph{Lookup (C0).} No arithmetic. The answer is read off once the relevant row or cell is located, e.g., \emph{``Which department does employee $x$ belong to?''} C0 is the cleanest probe of structural reasoning in isolation: any failure is attributable to incorrect row identification.

\paragraph{Counting (C1).} The answer is the number of rows satisfying a condition, e.g., \emph{``How many transactions were placed by customers over 30?''} Counting maps naturally to degree statistics on value nodes in \method{}'s graph but can be unreliable for LLMs on serialized text as the matching set grows.

\paragraph{Single aggregation (C2).} A single arithmetic operation over a set of values: sum, max, min, or mean, e.g., \emph{``What is the average age of employees in the engineering department?''} These operations require exact numeric values. Because \method{}'s graph encodes numeric values as quantile buckets, the graph alone cannot perform exact arithmetic; the exact values must be recovered from the textual serialization.

\paragraph{Multi-step derivation (C3).} Chained arithmetic where the output of one operation feeds the next, e.g., \emph{``By how much does the average order value of returning customers exceed that of new customers?''} C3 queries compound structural and arithmetic errors: a failure may stem from incorrect row selection, from arithmetic error at any step, or from both simultaneously.

\subsection{Using the Taxonomy}
\label{app:taxonomy-usage}

The stress-test in Appendix~\ref{app:stress-test} uses this taxonomy as an experimental scaffold. It holds the answer-type axis fixed (queries are A2 or A3) and varies the structural and computational axes independently, so that each empirical result can be attributed to a single source.

Three diagnostic recipes follow from the orthogonality of the axes:
\begin{itemize}
    \item \textbf{Fix S, vary C.} Holds the structural access pattern constant while increasing arithmetic demand. Failures isolate the arithmetic ceiling of the model and are not attributable to incorrect row identification.
    \item \textbf{Fix C, vary S.} Holds the computation constant (typically at C0, pure lookup) while increasing the structural depth required to find the right rows. Failures isolate structural reasoning capacity.
    \item \textbf{Match $(S, C)$, compare with/without \method{}.} At matched cells of the $(S, C)$ grid, the difference between the serialized baseline and \method{} isolates the contribution of the graph encoder.
\end{itemize}

A characteristic signature follows: if the graph encoder improves performance at C0 across S levels but not at C2--C3, the structural token is locating the right evidence, but arithmetic remains the binding constraint. Appendix~\ref{app:stress-test} reports exactly this.

%% file: appendix/app_stress-test.tex
\section{Fine-Grained Error Analysis and Stress-Test Design}
\label{app:stress-test}

Building on the taxonomy in Appendix~\ref{app:structural-hierarchy}, we now apply it as an experimental scaffold. The stress-test holds the answer-type axis fixed and varies the structural and computational axes jointly, generating a controlled grid of questions over which each failure mode can be attributed to a specific axis. We begin by stating two predictions about \method{}'s expected behavior on this grid, then describe the design and analyze the results.

\paragraph{Prediction 1: arithmetic ceiling at C2.}
\method{}'s graph constructor encodes numeric values as quantile buckets, not as exact numeric content. The graph therefore cannot perform exact arithmetic; it can only isolate the correct rows over which arithmetic must be applied by the LLM. We predict that pure aggregation queries (C2, especially \textsc{Avg}) will fail at low structural depth (S1, no filtering required), since there is no structural bottleneck for \method{} to relieve — only the arithmetic step remains, and that step is unchanged by the presence of the graph. The clearest case is the unconditioned \textsc{Avg} query: the model must sum a set of values and divide by their count, two operations that LLMs perform unreliably on serialized input regardless of how cleanly the input is presented. Conversely, when filtering reduces the value set, \method{} should help indirectly by shrinking the operand set the LLM must aggregate over.

\paragraph{Prediction 2: structural gain at S2--S3.}
Multiple simultaneous conditions on different columns (S3) require the model to identify rows satisfying joint constraints. On a serialized table, this demands matching patterns distributed across distant positions in the sequence; on \method{}'s graph, the conditions correspond to explicit incidence edges from a row node to typed value nodes. We predict that the graph encoder will help most on (S2--S3, C0--C1) cells: queries where locating the right rows is the dominant subtask and no exact arithmetic is required. 
Note that the flattening of hierarchical column headers in \method{}'s graph construction (where a nested header such as \emph{Export $>$ 2020 $>$ Q1} is expanded into separate columns) means that filtering on a deeply nested cell is structurally equivalent to applying multiple column conditions simultaneously, so both manifest as the same pattern in the graph.

\subsection{Stress-Test Set Design}
\label{app:stress-test-design}

We construct a controlled set of questions that varies the structural and computational axes independently, so that each cell of the $(S, C)$ grid can be attributed to a specific source of difficulty. A model that fails on \textsc{Avg} queries, for instance, may be failing because it cannot identify the correct rows (structural), because it cannot compute the mean over correctly identified rows (arithmetic), or both; varying one axis at a time disentangles these explanations.

\paragraph{Tables.}
The stress-test is conducted on 15 standard relational tables constructed in the style of HCT-QA. Concretely, the tables span the same domains as the real-world HCT-QA sources: scientific paper benchmarks, government statistics, and demographic census reports, and follow the same structural conventions as the HCT-QA synthetic generator: a flat relational layout where categorical attributes define row identity and numerical attributes fill the value columns. All tables contain at least three categorical columns and two numerical columns, a requirement imposed to support the increasing levels of structural difficulty examined below. Tables from the original HCT-QA pool were not used directly because too few contained the minimum number of categorical columns needed for the multi-condition experiments; the synthetic tables are otherwise in-distribution with the HCT-QA data the model was trained on.

\paragraph{Design.}
Questions are generated by independently varying two axes:

\begin{itemize}
  \item \textbf{Structural axis (S1--S3).} The number of simultaneous filter conditions applied before computing the answer: no conditions (global, over all rows; S1), one condition (single categorical filter; S2), or two conditions (joint filter on two distinct categorical columns; S3), optionally followed by a group-by operator partitioning rows by one or two categorical keys (G1 and G2 respectively).
  \item \textbf{Computational axis (C0--C2).} The aggregation operator applied to the retrieved values: \textsc{Lookup} (C0), \textsc{Count} (C1), \textsc{Max} (C2), and \textsc{Avg} (C2).
\end{itemize}

This design, extended with group-by variants, yields 2{,}337 questions across 15 tables and isolates each source of difficulty independently, allowing failures to be attributed to their origin rather than conflating structural and arithmetic errors.

\paragraph{Arithmetic bottleneck (verifies Prediction 1).}
Table~\ref{tab:stress_test} reveals a clear operator hierarchy. \textsc{Lookup} (C0) is the strongest category across all models and filter levels, with the serialized baseline achieving F1 of 66--74 regardless of the number of conditions. \textsc{Max} (C2) occupies a middle ground (F1\,$=$\,52--69), while \textsc{Avg} (C2) is the hardest operator by a large margin: both models score near zero on unconditioned average queries (F1\,$=$\,1.23 at (S1, C2)), where no filtering is required and arithmetic alone must be performed. This rules out a structural explanation for average failures and identifies exact arithmetic as a persistent limitation of LLMs on serialized tables, one that \textsc{+GRAB} cannot fully resolve either, exactly as predicted by the quantile-bucket encoding of value nodes.

\textsc{Count} (C1) shows a qualitatively different pattern: the serialized baseline performs poorly (F1\,$=$\,8.63 with no condition), while \textsc{+GRAB} achieves 54.90 --- the largest absolute gain in the entire table ($+$46.27 F1). This aligns with the natural capacity of message-passing GNNs to aggregate over neighborhoods, which maps directly onto counting over row-conditioned subgraphs.

\paragraph{Filtering makes questions easier, not harder.}
A consistent and counterintuitive pattern emerges across all operators: adding filter conditions \emph{reduces} difficulty rather than increasing it. For \textsc{Lookup}, performance improves monotonically from single (F1\,$=$\,66.59) to triple conditions (F1\,$=$\,73.90) on the serialized baseline. This is because more conditions uniquely identify the target row more precisely, reducing ambiguity in the answer. The LLM benefits from additional anchors in the text: each extra condition narrows the search space and makes the correct row easier to locate by pattern-matching over the serialized sequence.

The same effect holds for \textsc{Avg}: performance rises from F1\,$=$\,1.23 with no filter to 28.62 with two conditions, because filtering also reduces the number of values over which the mean must be computed, easing the arithmetic load. For \textsc{Count} and \textsc{Max}, the trend is less pronounced but consistent in direction. This finding has a practical implication: the stress-test categories without any filter condition are structurally the hardest, not the easiest, contrary to the intuition that more conditions imply more complexity.

\method{} amplifies this effect (verifying Prediction 2). The graph encoder provides column-typed value nodes that act as precise structural anchors, allowing the model to locate target cells even more reliably than pattern-matching over flat text. The gain is largest on \textsc{Lookup} ($+$14--24 F1 across S2--S3) and \textsc{Count} ($+$22--46 F1 across S1--S3), where locating the correct rows is the dominant subtask.

\paragraph{Group-by as the hardest structural probe.}
Group-by questions expose a qualitatively harder regime than any other category. Unlike lookup or aggregation, group-by requires the model to simultaneously partition the table by one or two categorical keys, apply an aggregation within each partition, and produce a complete structured answer with one tuple per group. This demands not only locating the relevant rows, but retaining and organizing them across multiple groups before producing the output: a form of working memory over the table that serialized text processing handles poorly.

The serialized baseline drops to F1\,$=$\,42.92 on Group-by \textsc{Max} single key and CC\,$=$\,0.00 on all double-key variants, meaning the model essentially never produces a fully correct multi-group answer. \method{} recovers substantially: F1\,$=$\,72.59 on Group-by \textsc{Max} single key ($+$29.67) with CC rising from 4.76 to 48.57, indicating that the graph encoder helps produce \emph{complete} structured outputs rather than only partial ones.

The counterintuitive filtering effect reappears here too. Adding a filter to a group-by query does not consistently increase difficulty and in some cases reduces it, presumably because the filter constrains the set of groups the model must track, reducing the working memory demand. Group-by \textsc{Avg} remains the hardest setting overall, since it compounds the group-tracking difficulty with exact arithmetic; \method{} reaches at most F1\,$=$\,56.59 here, consistent with the arithmetic ceiling identified in the scalar setting.

\input{tables/tab_stress-test-no-gpt}

\subsection{Taxonomy Verification}
\label{app:taxonomy-verification}

We close by mapping the results back onto the diagnostic recipes in Appendix~\ref{app:taxonomy-usage}, then verifying that the observed pattern is not an artifact of model scale.

\paragraph{Fixing S, varying C.} At S1 (no filtering), the operator hierarchy is C0\,$>$\,C1\,$>$\,C2 for both the serialized baseline and \textsc{+GRAB}. \textsc{Avg} at (S1, C2) is the unique cell where neither model exceeds F1\,$=$\,1.23, isolating exact arithmetic as the binding constraint when no structural reasoning is required. This is Prediction 1 confirmed.

\paragraph{Fixing C, varying S.} At C0, the serialized baseline improves with more filters (66.59 $\to$ 73.90 F1) and \textsc{+GRAB} improves further (80.72 $\to$ 95.52), confirming that structural localization is the binding constraint when no arithmetic is required. The gap between baseline and \textsc{+GRAB} widens at higher S, indicating that the graph's incidence structure provides increasing relative value as the number of joint conditions grows --- Prediction 2 confirmed.

\paragraph{Matched $(S, C)$, with vs.\ without \method{}.} \method{}'s gains concentrate at (S2--S3, C0--C1) cells, where the graph encoder relieves a real structural bottleneck without competing with arithmetic demand. Gains shrink at (S1, C2), where no structural bottleneck exists and only arithmetic remains. This is the characteristic signature predicted in Appendix~\ref{app:taxonomy-usage}: the structural token is locating the right evidence, and arithmetic remains the binding constraint where it appears.

\paragraph{Robustness to model scale.} We change the underlying LLM from Qwen3-4B to Qwen3-14B and rerun the stress-test on both the serialized baseline and \textsc{+GRAB} (Table~\ref{tab:stress_test_14b}). Three patterns confirm that the gains attributed to \method{} reflect a structural rather than a capacity bottleneck. First, the overall \textsc{+GRAB} gain \emph{widens} with scale, from $+21.42$ F1 at 4B to $+31.68$ F1 at 14B: a stronger backbone does not absorb the structural signal, it amplifies the benefit of receiving it in encoded form. Second, the arithmetic ceiling at (S1, C2) is preserved exactly --- $1.23$ F1 across all four 4B/14B $\times$ Serialized/\textsc{+GRAB} configurations --- exactly as predicted by the quantile-bucket encoding of value nodes. Third, the largest 14B gains concentrate precisely where the diagnostic recipes predict. On \textsc{Count}, \textsc{+GRAB} adds $+54.90$, $+51.43$, and $+26.98$ F1 across the three condition depths, since counting still maps onto value-node degree regardless of backbone size. On group-by, the contrast is sharper: the 14B serialized baseline actually \emph{regresses} relative to 4B on several Group-by \textsc{Count} settings (e.g., single key $34.58 \to 24.65$ F1, double key $37.44 \to 18.70$), consistent with the working-memory failure mode being orthogonal to model scale, while \textsc{+GRAB} recovers above $73$ F1 in every Group-by \textsc{Count} cell and lifts CC on Group-by \textsc{Max} single key by $+65.72$ points. Taken together, the predicted (S, C) signature reproduces at scale, and the graph encoder contributes along an axis of difficulty that additional parameters do not resolve.

\input{tables/tab_stress-qwen14b}

\section{Licenses}
\label{app:license}
We indicate the licenses of the artifacts used in this work, based on the official repositories, dataset cards, or release pages whenever available. All the artifacts used in this paper can be used for research: HiTab (Computational Use of Data Agreement v1.0), WikiTableQuestions (CC BY-SA 4.0), WikiSQL (BSD-3-Clause repository license), HCT-QA (MIT), TabMWP (CC BY-NC-SA 4.0), MultiHierTT (MIT), TQABench (GPL-3.0 license), GeoQuery (GPL-2.0), Spider (CC BY-SA 4.0), TabFact (CC BY 4.0), Qwen models (Apache 2.0), MultiTabQa (MIT), TableLLama (MIT).

%% file: tables/tab_stress-test-no-gpt.tex
\begin{table*}[t]
\centering
\caption{Stress-test results. F1 and CC scores averaged over questions across 15 different tables. +Prompt Tuning uses learned soft prompt vectors without a graph encoder; +GRAB adds the full graph encoder and query-conditioned latent bridge; Serialized only uses the frozen LLM with flattened table input. $\Delta$ reports the gain of +GRAB over Serialized only.}
\label{tab:stress_test}
\setlength{\tabcolsep}{4pt}
\renewcommand{\arraystretch}{1.05}
\small
\begin{tabular}{llcccccccc}
\toprule
& & \multicolumn{2}{c}{\textbf{Serialized only}} & \multicolumn{2}{c}{\textbf{+Prompt Tuning}} & \multicolumn{2}{c}{\textbf{+GRAB}} & \multicolumn{2}{c}{\textbf{$\Delta$}} \\
\cmidrule(lr){3-4} \cmidrule(lr){5-6} \cmidrule(lr){7-8} \cmidrule(lr){9-10}
\textbf{Question type} & \textbf{Cond.} & F1 & CC & F1 & CC & F1 & CC & $\Delta$F1 & $\Delta$CC \\
\midrule
Overall \small{(2337)}   & ---               & 39.29 & 22.59 & 46.39 & 21.69 & 60.71 & 39.50 & +21.42 & +16.91 \\
\midrule
\multicolumn{10}{l}{\textit{Lookup}} \\
\quad Single condition    & AND $\times$1     & 66.59 & 37.14 & 60.49 & 20.95 & 80.72 & 58.10 & +14.13 & +20.96 \\
\quad Double condition    & AND $\times$2     & 65.56 & 54.29 & 72.96 & 50.48 & 89.32 & 74.29 & +23.76 & +20.00 \\
\quad Triple condition    & AND $\times$3     & 73.90 & 87.62 & 82.63 & 86.67 & 95.52 & 92.38 & +21.62 & +4.76  \\
\midrule
\multicolumn{10}{l}{\textit{Count}} \\
\quad No condition        & ---               &  8.63 &  9.80 &  3.92 &  3.92 & 54.90 & 54.90 & +46.27 & +45.10 \\
\quad Single condition    & AND $\times$1     & 23.01 & 24.76 &  5.71 &  5.71 & 54.29 & 54.29 & +31.28 & +29.53 \\
\quad Double condition    & AND $\times$2     & 34.69 & 40.00 &  8.57 &  8.57 & 57.14 & 57.14 & +22.45 & +17.14 \\
\midrule
\multicolumn{10}{l}{\textit{Max}} \\
\quad No condition        & ---               & 57.67 & 62.96 & 56.79 & 56.79 & 93.83 & 93.83 & +36.16 & +30.87 \\
\quad Single condition    & AND $\times$1     & 52.21 & 56.19 & 63.81 & 63.81 & 85.71 & 85.71 & +33.50 & +29.52 \\
\quad Double condition    & AND $\times$2     & 68.76 & 75.24 & 68.57 & 68.57 & 79.05 & 79.05 & +10.29 & +3.81  \\
\midrule
\multicolumn{10}{l}{\textit{Avg}} \\
\quad No condition        & ---               &  1.23 &  1.23 &  1.23 &  1.23 &  1.23 &  1.23 &  +0.00 &  +0.00 \\
\quad Single condition    & AND $\times$1     &  9.38 & 14.29 &  8.57 &  8.57 & 15.24 & 15.24 &  +5.86 &  +0.95 \\
\quad Double condition    & AND $\times$2     & 28.62 & 34.29 & 22.86 & 22.86 & 40.95 & 40.95 & +12.33 &  +6.66 \\
\midrule
\multicolumn{10}{l}{\textit{Group-by Max}} \\
\quad Single key          & G1                & 42.92 &  4.76 & 66.48 & 24.76 & 72.59 & 48.57 & +29.67 & +43.81 \\
\quad Single key + filter & G1, AND $\times$1 & 43.69 &  6.67 & 71.73 & 40.00 & 70.20 & 37.14 & +26.51 & +30.47 \\
\quad Double key          & G2                & 38.28 &  0.00 & 52.96 &  1.90 & 60.00 & 17.14 & +21.72 & +17.14 \\
\quad Double key + filter & G2, AND $\times$1 & 38.58 &  0.00 & 63.89 &  3.81 & 55.88 & 12.38 & +17.30 & +12.38 \\
\midrule
\multicolumn{10}{l}{\textit{Group-by Count}} \\
\quad Single key          & G1                & 34.58 &  2.08 & 49.83 &  8.33 & 62.40 & 30.21 & +27.82 & +28.13 \\
\quad Single key + filter & G1, AND $\times$1 & 39.35 &  1.90 & 57.81 & 11.43 & 62.46 & 25.71 & +23.11 & +23.81 \\
\quad Double key          & G2                & 37.44 &  0.00 & 50.47 &  0.00 & 62.57 & 11.46 & +25.13 & +11.46 \\
\quad Double key + filter & G2, AND $\times$1 & 39.70 &  0.00 & 51.52 &  0.00 & 49.84 &  7.14 & +10.14 &  +7.14 \\
\midrule
\multicolumn{10}{l}{\textit{Group-by Avg}} \\
\quad Single key          & G1                & 21.63 &  1.90 & 28.95 &  1.90 & 38.77 &  6.67 & +17.14 &  +4.77 \\
\quad Single key + filter & G1, AND $\times$1 & 30.12 &  6.67 & 35.19 &  7.62 & 53.64 & 17.14 & +23.52 & +10.47 \\
\quad Double key          & G2                & 27.44 &  0.00 & 45.66 &  0.00 & 49.02 &  1.90 & +21.58 &  +1.90 \\
\quad Double key + filter & G2, AND $\times$1 & 38.46 &  0.95 & 56.79 &  0.95 & 56.59 & 14.29 & +18.13 & +13.34 \\
\bottomrule
\end{tabular}
\end{table*}

%% file: tables/tab_stress-qwen14b.tex
\begin{table*}[t]
\centering
\caption{Stress-test results across two backbones (Qwen3-4B and Qwen3-14B). F1 and CC scores averaged over questions across 15 different tables. Serialized only uses the frozen LLM with flattened table input; +GRAB adds the full graph encoder and query-conditioned latent bridge. $\Delta$F1 reports the F1 gain of +GRAB over Serialized only within each backbone.}
\label{tab:stress_test_14b}
\setlength{\tabcolsep}{4pt}
\renewcommand{\arraystretch}{1.05}
\small
\begin{tabular}{ll ccccc | ccccc}
\toprule
& & \multicolumn{5}{c|}{\textbf{Qwen3-4B}} & \multicolumn{5}{c}{\textbf{Qwen3-14B}} \\
\cmidrule(lr){3-7} \cmidrule(lr){8-12}
& & \multicolumn{2}{c}{Serialized} & \multicolumn{2}{c}{+GRAB} & 
& \multicolumn{2}{c}{Serialized} & \multicolumn{2}{c}{+GRAB} & \\
\cmidrule(lr){3-4} \cmidrule(lr){5-6} \cmidrule(lr){8-9} \cmidrule(lr){10-11}
\textbf{Question type} & \textbf{Cond.} & F1 & CC & F1 & CC & $\Delta$F1 & F1 & CC & F1 & CC & $\Delta$F1 \\
\midrule
Overall \small{(2337)} & --- & 39.29 & 22.59 & 60.71 & 39.50 & +21.42 & 40.65 & 27.04 & 72.33 & 52.20 & +31.68 \\
\midrule
\multicolumn{12}{l}{\textit{Lookup}} \\
\quad Single condition    & AND $\times$1 & 66.59 & 37.14 & 80.72 & 58.10 & +14.13 & 76.44 & 49.52 & 83.82 & 64.76 &  +7.38 \\
\quad Double condition    & AND $\times$2 & 65.56 & 54.29 & 89.32 & 74.29 & +23.76 & 77.74 & 64.76 & 90.61 & 78.10 & +12.87 \\
\quad Triple condition    & AND $\times$3 & 73.90 & 87.62 & 95.52 & 92.38 & +21.62 & 81.71 & 88.57 & 96.76 & 95.24 & +15.05 \\
\midrule
\multicolumn{12}{l}{\textit{Count}} \\
\quad No condition        & ---           &  8.63 &  9.80 & 54.90 & 54.90 & +46.27 & 21.57 & 21.57 & 76.47 & 76.47 & +54.90 \\
\quad Single condition    & AND $\times$1 & 23.01 & 24.76 & 54.29 & 54.29 & +31.28 & 27.62 & 27.62 & 79.05 & 79.05 & +51.43 \\
\quad Double condition    & AND $\times$2 & 34.69 & 40.00 & 57.14 & 57.14 & +22.45 & 56.83 & 57.14 & 83.81 & 83.81 & +26.98 \\
\midrule
\multicolumn{12}{l}{\textit{Max}} \\
\quad No condition        & ---           & 57.67 & 62.96 & 93.83 & 93.83 & +36.16 & 84.03 & 85.19 & 96.30 & 96.30 & +12.27 \\
\quad Single condition    & AND $\times$1 & 52.21 & 56.19 & 85.71 & 85.71 & +33.50 & 72.31 & 79.05 & 94.29 & 94.29 & +21.98 \\
\quad Double condition    & AND $\times$2 & 68.76 & 75.24 & 79.05 & 79.05 & +10.29 & 73.65 & 74.29 & 87.62 & 87.62 & +13.97 \\
\midrule
\multicolumn{12}{l}{\textit{Avg}} \\
\quad No condition        & ---           &  1.23 &  1.23 &  1.23 &  1.23 &  +0.00 &  1.23 &  1.23 &  1.23 &  1.23 &  +0.00 \\
\quad Single condition    & AND $\times$1 &  9.38 & 14.29 & 15.24 & 15.24 &  +5.86 &  9.71 & 10.48 & 17.14 & 17.14 &  +7.43 \\
\quad Double condition    & AND $\times$2 & 28.62 & 34.29 & 40.95 & 40.95 & +12.33 & 34.41 & 36.19 & 52.38 & 52.38 & +17.97 \\
\midrule
\multicolumn{12}{l}{\textit{Group-by Max}} \\
\quad Single key          & G1                & 42.92 &  4.76 & 72.59 & 48.57 & +29.67 & 36.49 &  8.57 & 86.59 & 74.29 & +50.10 \\
\quad Single key + filter & G1, AND $\times$1 & 43.69 &  6.67 & 70.20 & 37.14 & +26.51 & 41.81 & 11.43 & 86.15 & 62.86 & +44.34 \\
\quad Double key          & G2                & 38.28 &  0.00 & 60.00 & 17.14 & +21.72 & 39.68 &  2.86 & 72.67 & 19.05 & +32.99 \\
\quad Double key + filter & G2, AND $\times$1 & 38.58 &  0.00 & 55.88 & 12.38 & +17.30 & 37.52 &  0.95 & 80.32 & 39.05 & +42.80 \\
\midrule
\multicolumn{12}{l}{\textit{Group-by Count}} \\
\quad Single key          & G1                & 34.58 &  2.08 & 62.40 & 30.21 & +27.82 & 24.65 &  6.25 & 73.50 & 46.88 & +48.85 \\
\quad Single key + filter & G1, AND $\times$1 & 39.35 &  1.90 & 62.46 & 25.71 & +23.11 & 20.22 &  1.90 & 78.99 & 50.48 & +58.77 \\
\quad Double key          & G2                & 37.44 &  0.00 & 62.57 & 11.46 & +25.13 & 18.70 &  1.04 & 78.17 & 35.42 & +59.47 \\
\quad Double key + filter & G2, AND $\times$1 & 39.70 &  0.00 & 49.84 &  7.14 & +10.14 & 20.62 &  0.00 & 78.13 & 38.10 & +57.51 \\
\midrule
\multicolumn{12}{l}{\textit{Group-by Avg}} \\
\quad Single key          & G1                & 21.63 &  1.90 & 38.77 &  6.67 & +17.14 & 11.95 &  0.00 & 46.66 &  7.62 & +34.71 \\
\quad Single key + filter & G1, AND $\times$1 & 30.12 &  6.67 & 53.64 & 17.14 & +23.52 & 27.78 &  3.81 & 65.83 & 27.62 & +38.05 \\
\quad Double key          & G2                & 27.44 &  0.00 & 49.02 &  1.90 & +21.58 & 24.49 &  0.00 & 55.65 &  3.81 & +31.16 \\
\quad Double key + filter & G2, AND $\times$1 & 38.46 &  0.95 & 56.59 & 14.29 & +18.13 & 30.26 &  0.95 & 69.27 & 21.90 & +39.01 \\
\bottomrule
\end{tabular}
\end{table*}

%% file: custom.bib
@inproceedings{Xu+2018b,
title={How Powerful are Graph Neural Networks?},
author={Keyulu Xu and Weihua Hu and Jure Leskovec and Stefanie Jegelka},
booktitle={International Conference on Learning Representations},
year={2019},
url={https://openreview.net/forum?id=ryGs6iA5Km},
}

@inproceedings{Mor+2019,
author = {Morris, Christopher and Ritzert, Martin and Fey, Matthias and Hamilton, William L. and Lenssen, Jan Eric and Rattan, Gaurav and Grohe, Martin},
title = {Weisfeiler and {L}eman go neural: higher-order graph neural networks},
year = {2019},
isbn = {978-1-57735-809-1},
publisher = {AAAI Press},
url = {https://doi.org/10.1609/aaai.v33i01.33014602},
doi = {10.1609/aaai.v33i01.33014602},
booktitle = {Proceedings of the Thirty-Third AAAI Conference on Artificial Intelligence and Thirty-First Innovative Applications of Artificial Intelligence Conference and Ninth AAAI Symposium on Educational Advances in Artificial Intelligence},
articleno = {565},
numpages = {8},
location = {Honolulu, Hawaii, USA},
series = {AAAI'19/IAAI'19/EAAI'19}
}

@inproceedings{
huang2024lorahub,
title={LoraHub: Efficient Cross-Task Generalization via Dynamic Lo{RA} Composition},
author={Chengsong Huang and Qian Liu and Bill Yuchen Lin and Tianyu Pang and Chao Du and Min Lin},
booktitle={First Conference on Language Modeling},
year={2024},
url={https://openreview.net/forum?id=TrloAXEJ2B}
}

@inproceedings{HuSWALWWC22,
  author       = {Edward J. Hu and
                  Yelong Shen and
                  Phillip Wallis and
                  Zeyuan Allen{-}Zhu and
                  Yuanzhi Li and
                  Shean Wang and
                  Lu Wang and
                  Weizhu Chen},
  title        = {LoRA: Low-Rank Adaptation of Large Language Models},
  booktitle    = {The Tenth International Conference on Learning Representations, {ICLR}
                  2022, Virtual Event, April 25-29, 2022},
  publisher    = {OpenReview.net},
  year         = {2022},
  url          = {https://openreview.net/forum?id=nZeVKeeFYf9},
  bibsource    = {dblp computer science bibliography, https://dblp.org}
}

@inproceedings{xgboost,
author = {Chen, Tianqi and Guestrin, Carlos},
title = {XGBoost: A Scalable Tree Boosting System},
year = {2016},
isbn = {9781450342322},
publisher = {Association for Computing Machinery},
address = {New York, NY, USA},
url = {https://doi.org/10.1145/2939672.2939785},
doi = {10.1145/2939672.2939785},
booktitle = {Proceedings of the 22nd ACM SIGKDD International Conference on Knowledge Discovery and Data Mining},
pages = {785–794},
numpages = {10},
location = {San Francisco, California, USA},
series = {KDD '16}
}

@inproceedings{shen24,
    title = "Improving Retrieval-augmented Text-to-{SQL} with {AST}-based Ranking and Schema Pruning",
    author = "Shen, Zhili  and
      Vougiouklis, Pavlos  and
      Diao, Chenxin  and
      Vyas, Kaustubh  and
      Ji, Yuanyi  and
      Pan, Jeff Z.",
    editor = "Al-Onaizan, Yaser  and
      Bansal, Mohit  and
      Chen, Yun-Nung",
    booktitle = "Proceedings of the 2024 Conference on Empirical Methods in Natural Language Processing",
    month = nov,
    year = "2024",
    address = "Miami, Florida, USA",
    publisher = "Association for Computational Linguistics",
    url = "https://aclanthology.org/2024.emnlp-main.449/",
    doi = "10.18653/v1/2024.emnlp-main.449",
    pages = "7865--7879"
}

@article{aixel,
author = {Zhang, Chi and Zhang, Meihui and Yang, Yuxin and Chen, Tao and Luo, Zhaojing},
title = {AixelAsk: A Stepwise-Guided Retrieval and Reasoning Framework for Large Table QA},
year = {2025},
issue_date = {December 2025},
publisher = {Association for Computing Machinery},
address = {New York, NY, USA},
volume = {3},
number = {6},
url = {https://doi.org/10.1145/3769831},
doi = {10.1145/3769831},
journal = {Proc. ACM Manag. Data},
month = dec,
articleno = {366},
numpages = {25}
}

@inproceedings{unifiedskg,
    title = "{U}nified{SKG}: Unifying and Multi-Tasking Structured Knowledge Grounding with Text-to-Text Language Models",
    author = "Xie, Tianbao  and
      Wu, Chen Henry  and
      Shi, Peng  and
      Zhong, Ruiqi  and
      Scholak, Torsten  and
      Yasunaga, Michihiro  and
      Wu, Chien-Sheng  and
      Zhong, Ming  and
      Yin, Pengcheng  and
      Wang, Sida I.  and
      Zhong, Victor  and
      Wang, Bailin  and
      Li, Chengzu  and
      Boyle, Connor  and
      Ni, Ansong  and
      Yao, Ziyu  and
      Radev, Dragomir  and
      Xiong, Caiming  and
      Kong, Lingpeng  and
      Zhang, Rui  and
      Smith, Noah A.  and
      Zettlemoyer, Luke  and
      Yu, Tao",
    editor = "Goldberg, Yoav  and
      Kozareva, Zornitsa  and
      Zhang, Yue",
    booktitle = "Proceedings of the 2022 Conference on Empirical Methods in Natural Language Processing",
    month = dec,
    year = "2022",
    address = "Abu Dhabi, United Arab Emirates",
    publisher = "Association for Computational Linguistics",
    url = "https://aclanthology.org/2022.emnlp-main.39/",
    doi = "10.18653/v1/2022.emnlp-main.39",
    pages = "602--631"
}

@inproceedings{tableformer,
    title = "{T}able{F}ormer: Robust Transformer Modeling for Table-Text Encoding",
    author = "Yang, Jingfeng  and
      Gupta, Aditya  and
      Upadhyay, Shyam  and
      He, Luheng  and
      Goel, Rahul  and
      Paul, Shachi",
    editor = "Muresan, Smaranda  and
      Nakov, Preslav  and
      Villavicencio, Aline",
    booktitle = "Proceedings of the 60th Annual Meeting of the Association for Computational Linguistics (Volume 1: Long Papers)",
    month = may,
    year = "2022",
    address = "Dublin, Ireland",
    publisher = "Association for Computational Linguistics",
    url = "https://aclanthology.org/2022.acl-long.40/",
    doi = "10.18653/v1/2022.acl-long.40",
    pages = "528--537"
}

@inproceedings{chen-etal-2024,
    title = "Is Table Retrieval a Solved Problem? Exploring Join-Aware Multi-Table Retrieval",
    author = "Chen, Peter Baile  and
      Zhang, Yi  and
      Roth, Dan",
    editor = "Ku, Lun-Wei  and
      Martins, Andre  and
      Srikumar, Vivek",
    booktitle = "Proceedings of the 62nd Annual Meeting of the Association for Computational Linguistics (Volume 1: Long Papers)",
    month = aug,
    year = "2024",
    address = "Bangkok, Thailand",
    publisher = "Association for Computational Linguistics",
    url = "https://aclanthology.org/2024.acl-long.148/",
    doi = "10.18653/v1/2024.acl-long.148",
    pages = "2687--2699"
}

@inproceedings{
instructblip,
title={Instruct{BLIP}: Towards General-purpose Vision-Language Models with Instruction Tuning},
author={Wenliang Dai and Junnan Li and Dongxu Li and Anthony Tiong and Junqi Zhao and Weisheng Wang and Boyang Li and Pascale Fung and Steven Hoi},
booktitle={Thirty-seventh Conference on Neural Information Processing Systems},
year={2023},
url={https://openreview.net/forum?id=vvoWPYqZJA}
}

@inproceedings{
barcelo20,
title={The Logical Expressiveness of Graph Neural Networks},
author={Pablo Barceló and Egor V. Kostylev and Mikael Monet and Jorge Pérez and Juan Reutter and Juan Pablo Silva},
booktitle={International Conference on Learning Representations},
year={2020},
url={https://openreview.net/forum?id=r1lZ7AEKvB}
}

@inproceedings{gilmer17,
author = {Gilmer, Justin and Schoenholz, Samuel S. and Riley, Patrick F. and Vinyals, Oriol and Dahl, George E.},
title = {Neural message passing for Quantum chemistry},
year = {2017},
publisher = {JMLR.org},
booktitle = {Proceedings of the 34th International Conference on Machine Learning - Volume 70},
pages = {1263–1272},
numpages = {10},
location = {Sydney, NSW, Australia},
series = {ICML'17},
url ={https://proceedings.mlr.press/v70/gilmer17a/gilmer17a.pdf}
}

@inproceedings{liu22,
    title = "{P}-Tuning: Prompt Tuning Can Be Comparable to Fine-tuning Across Scales and Tasks",
    author = "Liu, Xiao  and
      Ji, Kaixuan  and
      Fu, Yicheng  and
      Tam, Weng  and
      Du, Zhengxiao  and
      Yang, Zhilin  and
      Tang, Jie",
    editor = "Muresan, Smaranda  and
      Nakov, Preslav  and
      Villavicencio, Aline",
    booktitle = "Proceedings of the 60th Annual Meeting of the Association for Computational Linguistics (Volume 2: Short Papers)",
    month = may,
    year = "2022",
    address = "Dublin, Ireland",
    publisher = "Association for Computational Linguistics",
    url = "https://aclanthology.org/2022.acl-short.8/",
    doi = "10.18653/v1/2022.acl-short.8",
    pages = "61--68"
}

@inproceedings{prefix,
    title = "Prefix-Tuning: Optimizing Continuous Prompts for Generation",
    author = "Li, Xiang Lisa  and
      Liang, Percy",
    editor = "Zong, Chengqing  and
      Xia, Fei  and
      Li, Wenjie  and
      Navigli, Roberto",
    booktitle = "Proceedings of the 59th Annual Meeting of the Association for Computational Linguistics and the 11th International Joint Conference on Natural Language Processing (Volume 1: Long Papers)",
    month = aug,
    year = "2021",
    address = "Online",
    publisher = "Association for Computational Linguistics",
    url = "https://aclanthology.org/2021.acl-long.353/",
    doi = "10.18653/v1/2021.acl-long.353",
    pages = "4582--4597"
}

@inproceedings{blip2,
author = {Li, Junnan and Li, Dongxu and Savarese, Silvio and Hoi, Steven},
title = {BLIP-2: bootstrapping language-image pre-training with frozen image encoders and large language models},
year = {2023},
publisher = {JMLR.org},
url = 	 {https://proceedings.mlr.press/v202/li23q.html},
booktitle = {Proceedings of the 40th International Conference on Machine Learning},
articleno = {814},
numpages = {13},
location = {Honolulu, Hawaii, USA},
series = {ICML'23}
}

@inproceedings{flamingo,
author = {Alayrac, Jean-Baptiste and Donahue, Jeff and Luc, Pauline and Miech, Antoine and Barr, Iain and Hasson, Yana and Lenc, Karel and Mensch, Arthur and Millicah, Katie and Reynolds, Malcolm and Ring, Roman and Rutherford, Eliza and Cabi, Serkan and Han, Tengda and Gong, Zhitao and Samangooei, Sina and Monteiro, Marianne and Menick, Jacob and Borgeaud, Sebastian and Brock, Andrew and Nematzadeh, Aida and Sharifzadeh, Sahand and Binkowski, Mikolaj and Barreira, Ricardo and Vinyals, Oriol and Zisserman, Andrew and Simonyan, Karen},
title = {Flamingo: a visual language model for few-shot learning},
year = {2022},
isbn = {9781713871088},
publisher = {Curran Associates Inc.},
address = {Red Hook, NY, USA},
booktitle = {Proceedings of the 36th International Conference on Neural Information Processing Systems},
articleno = {1723},
numpages = {21},
location = {New Orleans, LA, USA},
series = {NeurIPS '22},
url ={https://openreview.net/forum?id=EbMuimAbPbs}
}

@inproceedings{
relbench,
title={RelBench: A Benchmark for Deep Learning on Relational Databases},
author={Joshua Robinson and Rishabh Ranjan and Weihua Hu and Kexin Huang and Jiaqi Han and Alejandro Dobles and Matthias Fey and Jan Eric Lenssen and Yiwen Yuan and Zecheng Zhang and Xinwei He and Jure Leskovec},
booktitle={The Thirty-eight Conference on Neural Information Processing Systems Datasets and Benchmarks Track},
year={2024},
url={https://openreview.net/forum?id=WEFxOm3Aez}
}

@article{tabpfn,
  author       = {Noah Hollmann and
                  Samuel M{\"{u}}ller and
                  Lennart Purucker and
                  Arjun Krishnakumar and
                  Max K{\"{o}}rfer and
                  Shi Bin Hoo and
                  Robin Tibor Schirrmeister and
                  Frank Hutter},
  title        = {Accurate predictions on small data with a tabular foundation model},
  journal      = {Nat.},
  volume       = {637},
  number       = {8044},
  pages        = {319--326},
  year         = {2025},
  url          = {https://doi.org/10.1038/s41586-024-08328-6},
  doi          = {10.1038/S41586-024-08328-6},
  bibsource    = {dblp computer science bibliography, https://dblp.org}
}

@inproceedings{griqa,
    title = "{GRI}-{QA}: a Comprehensive Benchmark for Table Question Answering over Environmental Data",
    author = "Contalbo, Michele Luca  and
      Pederzoli, Sara  and
      Buono, Francesco Del  and
      Valeria, Venturelli  and
      Guerra, Francesco  and
      Paganelli, Matteo",
    editor = "Che, Wanxiang  and
      Nabende, Joyce  and
      Shutova, Ekaterina  and
      Pilehvar, Mohammad Taher",
    booktitle = "Findings of the Association for Computational Linguistics: ACL 2025",
    month = jul,
    year = "2025",
    address = "Vienna, Austria",
    publisher = "Association for Computational Linguistics",
    url = "https://aclanthology.org/2025.findings-acl.814/",
    doi = "10.18653/v1/2025.findings-acl.814",
    pages = "15764--15779",
    ISBN = "979-8-89176-256-5"
}

@article{badaro23,
    title = "Transformers for Tabular Data Representation: A Survey of Models and Applications",
    author = "Badaro, Gilbert  and
      Saeed, Mohammed  and
      Papotti, Paolo",
    journal = "Transactions of the Association for Computational Linguistics",
    volume = "11",
    year = "2023",
    address = "Cambridge, MA",
    publisher = "MIT Press",
    url = "https://aclanthology.org/2023.tacl-1.14/",
    doi = "10.1162/tacl_a_00544",
    pages = "227--249"
}

@inproceedings{tablellama,
    title = "{T}able{L}lama: Towards Open Large Generalist Models for Tables",
    author = "Zhang, Tianshu  and
      Yue, Xiang  and
      Li, Yifei  and
      Sun, Huan",
    editor = "Duh, Kevin  and
      Gomez, Helena  and
      Bethard, Steven",
    booktitle = "Proceedings of the 2024 Conference of the North American Chapter of the Association for Computational Linguistics: Human Language Technologies (Volume 1: Long Papers)",
    month = jun,
    year = "2024",
    address = "Mexico City, Mexico",
    publisher = "Association for Computational Linguistics",
    url = "https://aclanthology.org/2024.naacl-long.335/",
    doi = "10.18653/v1/2024.naacl-long.335",
    pages = "6024--6044"
}

@inproceedings{
    tapex,
    title={{TAPEX}: Table Pre-training via Learning a Neural {SQL} Executor},
    author={Qian Liu and Bei Chen and Jiaqi Guo and Morteza Ziyadi and Zeqi Lin and Weizhu Chen and Jian-Guang Lou},
    booktitle={International Conference on Learning Representations},
    year={2022},
    url={https://openreview.net/forum?id=O50443AsCP}
}

@inproceedings{tapas,
    title = "{T}a{P}as: Weakly Supervised Table Parsing via Pre-training",
    author = {Herzig, Jonathan  and
      Nowak, Pawel Krzysztof  and
      M{\"u}ller, Thomas  and
      Piccinno, Francesco  and
      Eisenschlos, Julian},
    editor = "Jurafsky, Dan  and
      Chai, Joyce  and
      Schluter, Natalie  and
      Tetreault, Joel",
    booktitle = "Proceedings of the 58th Annual Meeting of the Association for Computational Linguistics",
    month = jul,
    year = "2020",
    address = "Online",
    publisher = "Association for Computational Linguistics",
    url = "https://aclanthology.org/2020.acl-main.398/",
    doi = "10.18653/v1/2020.acl-main.398",
    pages = "4320--4333"
}

@misc{grables,
      title={Grables: Tabular Learning Beyond Independent Rows}, 
      author={Tamara Cucumides and Floris Geerts},
      year={2026},
      eprint={2602.03945},
      archivePrefix={arXiv},
      primaryClass={cs.LG},
      url={https://arxiv.org/abs/2602.03945}, 
}

@inproceedings{
tamo,
title={Table as a Modality for Large Language Models},
author={Liyao Li and Chao Ye and Wentao Ye and Yifei Sun and Zhe Jiang and Haobo Wang and Jiaming Tian and Yiming Zhang and NINGTAO WANG and Xing Fu and Gang Chen and Junbo Zhao},
booktitle={The Thirty-ninth Annual Conference on Neural Information Processing Systems},
year={2025},
url={https://openreview.net/forum?id=kurEZdWU9G}
}

@proceedings{trl25,
    title = "Proceedings of the 4th Table Representation Learning Workshop",
    editor = "Chang, Shuaichen  and
      Hulsebos, Madelon  and
      Liu, Qian  and
      Chen, Wenhu  and
      Sun, Huan",
    month = jul,
    year = "2025",
    address = "Vienna, Austria",
    publisher = "Association for Computational Linguistics",
    url = "https://aclanthology.org/2025.trl-1.0/",
    doi = "10.18653/v1/2025.trl-1.0",
    ISBN = "979-8-89176-268-8"
}

@inproceedings{cheng-etal-2022-hitab,
    title = "{H}i{T}ab: A Hierarchical Table Dataset for Question Answering and Natural Language Generation",
    author = "Cheng, Zhoujun  and
      Dong, Haoyu  and
      Wang, Zhiruo  and
      Jia, Ran  and
      Guo, Jiaqi  and
      Gao, Yan  and
      Han, Shi  and
      Lou, Jian-Guang  and
      Zhang, Dongmei",
    editor = "Muresan, Smaranda  and
      Nakov, Preslav  and
      Villavicencio, Aline",
    booktitle = "Proceedings of the 60th Annual Meeting of the Association for Computational Linguistics (Volume 1: Long Papers)",
    month = may,
    year = "2022",
    address = "Dublin, Ireland",
    publisher = "Association for Computational Linguistics",
    url = "https://aclanthology.org/2022.acl-long.78/",
    doi = "10.18653/v1/2022.acl-long.78",
    pages = "1094--1110"
}

@inproceedings{wtq,
  title={Compositional semantic parsing on semi-structured tables},
  author={Pasupat, Panupong and Liang, Percy},
  booktitle={Proceedings of the 53rd Annual Meeting of the Association for Computational Linguistics and the 7th International Joint Conference on Natural Language Processing (Volume 1: Long Papers)},
  pages={1470--1480},
  year={2015},
  url = {https://aclanthology.org/P15-1142/}
}

@misc{wikisql,
  title={Seq2SQL: Generating Structured Queries from Natural Language using Reinforcement Learning}, 
  author={Victor Zhong and Caiming Xiong and Richard Socher},
  year={2017},
  eprint={1709.00103},
  archivePrefix={arXiv},
  primaryClass={cs.CL},
  url={https://arxiv.org/abs/1709.00103}, 
}

@misc{hctqa,
      title={HCT-QA: A Benchmark for Question Answering on Human-Centric Tables}, 
      author={Mohammad S. Ahmad and Zan A. Naeem and Michaël Aupetit and Ahmed Elmagarmid and Mohamed Eltabakh and Xiaosong Ma and Mourad Ouzzani and Chaoyi Ruan and Hani Al-Sayeh},
      year={2026},
      eprint={2504.20047},
      archivePrefix={arXiv},
      primaryClass={cs.IR},
      url={https://arxiv.org/abs/2504.20047}, 
}

@inproceedings{tabmwp,
title={Dynamic Prompt Learning via Policy Gradient for Semi-structured Mathematical Reasoning},
author={Pan Lu and Liang Qiu and Kai-Wei Chang and Ying Nian Wu and Song-Chun Zhu and Tanmay Rajpurohit and Peter Clark and Ashwin Kalyan},
booktitle={The Eleventh International Conference on Learning Representations },
year={2023},
url={https://openreview.net/forum?id=DHyHRBwJUTN}
}

@inproceedings{zhang-etal-2025-scitat,
    title = "{SCITAT}: A Question Answering Benchmark for Scientific Tables and Text Covering Diverse Reasoning Types",
    author = "Zhang, Xuanliang  and
      Wang, Dingzirui  and
      Wang, Baoxin  and
      Dou, Longxu  and
      Lu, Xinyuan  and
      Xu, Keyan  and
      Wu, Dayong  and
      Zhu, Qingfu",
    editor = "Che, Wanxiang  and
      Nabende, Joyce  and
      Shutova, Ekaterina  and
      Pilehvar, Mohammad Taher",
    booktitle = "Findings of the Association for Computational Linguistics: ACL 2025",
    month = jul,
    year = "2025",
    address = "Vienna, Austria",
    publisher = "Association for Computational Linguistics",
    url = "https://aclanthology.org/2025.findings-acl.199/",
    doi = "10.18653/v1/2025.findings-acl.199",
    pages = "3859--3881",
    ISBN = "979-8-89176-256-5"
}

@inproceedings{mmqa,
 author = {Wu, Jian and Yang, Linyi and Li, Dongyuan and Ji, Yuliang and Okumura, Manabu and Zhang, Yue},
 booktitle = {International Conference on Learning Representations},
 editor = {Y. Yue and A. Garg and N. Peng and F. Sha and R. Yu},
 pages = {48626--48643},
 title = {MMQA: Evaluating LLMs with Multi-Table Multi-Hop Complex Questions},
 url = {https://proceedings.iclr.cc/paper_files/paper/2025/file/794a425a2e47e05d29d30f79b79a692d-Paper-Conference.pdf},
 volume = {2025},
 year = {2025}
}

@misc{qiu2024tqabenchevaluatingllmsmultitable,
      title={TQA-Bench: Evaluating LLMs for Multi-Table Question Answering with Scalable Context and Symbolic Extension}, 
      author={Zipeng Qiu and You Peng and Guangxin He and Binhang Yuan and Chen Wang},
      year={2024},
      eprint={2411.19504},
      archivePrefix={arXiv},
      primaryClass={cs.AI},
      url={https://arxiv.org/abs/2411.19504}, 
}

@inproceedings{pal-etal-2023-multitabqa,
    title = "{M}ulti{T}ab{QA}: Generating Tabular Answers for Multi-Table Question Answering",
    author = "Pal, Vaishali  and
      Yates, Andrew  and
      Kanoulas, Evangelos  and
      de Rijke, Maarten",
    editor = "Rogers, Anna  and
      Boyd-Graber, Jordan  and
      Okazaki, Naoaki",
    booktitle = "Proceedings of the 61st Annual Meeting of the Association for Computational Linguistics (Volume 1: Long Papers)",
    month = jul,
    year = "2023",
    address = "Toronto, Canada",
    publisher = "Association for Computational Linguistics",
    url = "https://aclanthology.org/2023.acl-long.348/",
    doi = "10.18653/v1/2023.acl-long.348",
    pages = "6322--6334"
}

@inproceedings{zhao-etal-2022-multihiertt,
    title = "{M}ulti{H}iertt: Numerical Reasoning over Multi Hierarchical Tabular and Textual Data",
    author = "Zhao, Yilun  and
      Li, Yunxiang  and
      Li, Chenying  and
      Zhang, Rui",
    booktitle = "Proceedings of the 60th Annual Meeting of the Association for Computational Linguistics (Volume 1: Long Papers)",
    month = may,
    year = "2022",
    address = "Dublin, Ireland",
    publisher = "Association for Computational Linguistics",
    url = "https://aclanthology.org/2022.acl-long.454",
    pages = "6588--6600",
}

@misc{openai2026gpt54mini,
  title        = {Introducing GPT-5.4 mini and nano},
  author       = {{OpenAI}},
  year         = {2026},
  month        = mar,
  howpublished = {\url{https://openai.com/index/introducing-gpt-5-4-mini-and-nano/}},
  note         = {Accessed: 2026-05-26}
}

@article{qwen3,
    title={Qwen3 Technical Report}, 
    author={An Yang and Anfeng Li and Baosong Yang and Beichen Zhang and Binyuan Hui and Bo Zheng and Bowen Yu and Chang Gao and Chengen Huang and Chenxu Lv and Chujie Zheng and Dayiheng Liu and Fan Zhou and Fei Huang and Feng Hu and Hao Ge and Haoran Wei and Huan Lin and Jialong Tang and Jian Yang and Jianhong Tu and Jianwei Zhang and Jianxin Yang and Jiaxi Yang and Jing Zhou and Jingren Zhou and Junyang Lin and Kai Dang and Keqin Bao and Kexin Yang and Le Yu and Lianghao Deng and Mei Li and Mingfeng Xue and Mingze Li and Pei Zhang and Peng Wang and Qin Zhu and Rui Men and Ruize Gao and Shixuan Liu and Shuang Luo and Tianhao Li and Tianyi Tang and Wenbiao Yin and Xingzhang Ren and Xinyu Wang and Xinyu Zhang and Xuancheng Ren and Yang Fan and Yang Su and Yichang Zhang and Yinger Zhang and Yu Wan and Yuqiong Liu and Zekun Wang and Zeyu Cui and Zhenru Zhang and Zhipeng Zhou and Zihan Qiu},
    journal = {arXiv preprint arXiv:2505.09388},
    year={2025},
    eprint={2505.09388},
    archivePrefix={arXiv},
    primaryClass={cs.CL},
    url={https://arxiv.org/abs/2505.09388}, 
}

@inproceedings{yin-etal-2020-tabert,
    title = "{T}a{BERT}: Pretraining for Joint Understanding of Textual and Tabular Data",
    author = "Yin, Pengcheng  and
      Neubig, Graham  and
      Yih, Wen-tau  and
      Riedel, Sebastian",
    editor = "Jurafsky, Dan  and
      Chai, Joyce  and
      Schluter, Natalie  and
      Tetreault, Joel",
    booktitle = "Proceedings of the 58th Annual Meeting of the Association for Computational Linguistics",
    month = jul,
    year = "2020",
    address = "Online",
    publisher = "Association for Computational Linguistics",
    url = "https://aclanthology.org/2020.acl-main.745/",
    doi = "10.18653/v1/2020.acl-main.745",
    pages = "8413--8426"
}
